\documentclass{article}

\usepackage[english]{babel}
\usepackage[utf8]{inputenc}
\usepackage[T1]{fontenc}
\usepackage[a4paper,left=2cm,right=2cm,top=2cm,bottom=3cm]{geometry}  

\usepackage{microtype}
\usepackage{graphicx}
\usepackage{booktabs} 

\usepackage{hyperref}

\usepackage{hyperref}   
\usepackage{makecell}   
\usepackage{url}            
\usepackage{booktabs}       
\usepackage{amsfonts}       
\usepackage{nicefrac}       
\usepackage{subcaption}
\usepackage{adjustbox}
\usepackage{multirow} 

\usepackage{microtype}      

\usepackage{multirow, makecell} 

\usepackage{amsmath, amssymb, amsthm, mathtools}
\usepackage{color}

\usepackage{bm, bbm, dsfont}

\newtheorem{theorem}{Theorem}
 
\newtheorem{definition}{Definition}
\newtheorem{lemma}{Lemma}

\newtheorem{proposition}{Proposition}
\newtheorem{assumption}{Assumption}

\newtheoremstyle{TheoremNum}
        {\topsep}{\topsep}              
        {\itshape}                      
        {}                              
        {\bfseries}                     
        {.}                             
        { }                             
        {\thmname{#1}\thmnote{ \bfseries #3}}
\theoremstyle{TheoremNum}
\newtheorem{theoremn}{Theorem}

\newtheoremstyle{LemmaNum}
        {\topsep}{\topsep}              
        {\itshape}                      
        {}                              
        {\bfseries}                     
        {.}                             
        { }                             
        {\thmname{#1}\thmnote{ \bfseries #3}}
\theoremstyle{LemmaNum}
\newtheorem{lemman}{Lemma}

\newtheoremstyle{PropositionNum}
        {\topsep}{\topsep}              
        {\itshape}                      
        {}                              
        {\bfseries}                     
        {.}                             
        { }                             
        {\thmname{#1}\thmnote{ \bfseries #3}}
\theoremstyle{PropositionNum}
\newtheorem{propositionn}{Proposition}

\newtheoremstyle{RemarkNum}
        {\topsep}{\topsep}              
        {\itshape}                      
        {}                              
        {\bfseries}                     
        {.}                             
        { }                             
        {\thmname{#1}\thmnote{ \bfseries #3}}
\theoremstyle{RemarkNum}

\newtheoremstyle{AssumptionNum}
        {\topsep}{\topsep}              
        {\itshape}                      
        {}                              
        {\bfseries}                     
        {.}                             
        { }                             
        {\thmname{#1}\thmnote{ \bfseries #3}}
\theoremstyle{AssumptionNum}

\newcommand{\C}{C_{\max}}

\renewcommand{\(}{\left(}
\renewcommand{\)}{\right)}

\newcommand{\G}[3]{\frac{ \C \sqrt{ {#2} {#1}  \log \frac{d}{{#3}} } } {  {#1}  \lambda_- -{#1}^{3/4} H \log \frac{m}{{#3}} - \C \sqrt{ {#2}  {#1}  \log \frac{d}{{#3}} } } }

\newcommand{\Kmin}{ 2 \max \left\{ \frac{ \left( H \log \frac{m}{\delta} \right)^4 }{\lambda_-^4} , \frac{ \C^2 H\log \frac{d}{\delta} }{\lambda_-^2} \right\}  }

\newcommand{\Gp}[1]{ m^2 C_w^2  \log \left( \delta^{-1}  \sqrt{ 1 + \frac{H{#1} C^2 }{ m }  }  \right)  }

\newcommand{\sumHk}{\sum_{ \substack{ h < H \\ k^\prime < k }}}

\newcommand{\Mk}{ V_k^\dagger \bar{Z}_k \left( \bar{X}_k^{next} \right)^\top }

\newcommand{\Vtil}[1]{  I_{d+d_u} +\sum_{h=1}^{H-1} \sum_{k^\prime = 1}^{{#1}-1} \hat{z}_{k^\prime,h} \(\hat{z}_{k^\prime,h} \)^\top }

\usepackage{enumitem} 

\usepackage{natbib}

\usepackage{algorithm}
\usepackage{algcompatible} 

\begin{document}

\title{Episodic Linear Quadratic Regulators with Low-rank Transitions}

\author{Tianyu Wang\footnote{tianyu@cs.duke.edu, Duke University.} \quad Lin F. Yang\footnote{linyang@ee.ucla.edu, UCLA.}}
\date{}

\maketitle

\begin{abstract}
    Linear Quadratic Regulators (LQR) achieve enormous successful real-world applications. Very recently, people have been focusing on efficient learning algorithms for LQRs when their dynamics are unknown. Existing results effectively learn to control the unknown system using number of episodes depending polynomially on the system parameters, including the ambient dimension of the states. These traditional approaches, however, become inefficient in common scenarios, e.g., when the states are high-resolution images. In this paper, we propose an algorithm that utilizes the intrinsic system low-rank structure for efficient learning. For problems of rank-$m$, our algorithm achieves a $K$-episode regret bound of order $\widetilde{O}(m^{3/2} K^{1/2})$. Consequently, the sample complexity of our algorithm only depends on the rank, $m$, rather than the ambient dimension, $d$, which can be orders-of-magnitude larger. 
\end{abstract}


\section{Introduction}
The classic problems of control and reinforcement learning have again captured considerable interests, following the recent successes in challenging domains like video games \citep{mnih2015human} and GO \citep{silver2016mastering}. A classic and charming control model is the Linear Quadratic Regulator (LQR) model. In an LQR problem, the system state transitions are linear, and the cost to be minimized is quadratic in states and control actions. 
Simply formulated, LQR models successfully solve many challenging and important real-world problems, e.g., autonomous aerial vehicle (AAV) control \citep{abbeel2007application}, 
 robotic arms \citep{li2004iterative, platt2010belief}, and humandroid control \citep{mason2014full}. 

In an LQR problem, the transition of states, $x \in \mathbb{R}^d$, depends linearly on the current state, the control action, $u \in \mathbb{R}^{d_u}$ played by the agent, and a noise $w \in \mathbb{R}^d$. Symbolically, the system evolves as $ x \leftarrow A x + B u + w $, where $A$ and $B$ are matrices that describe the system dynamics. 
Every time the agent executes an action, an immediate quadratic cost (over the state and actions) is incurred, and the system transits to the next state. The goal of the agent is to find a policy that minimizes the expected total cost (in a period of time starting from any state). 

If the dynamics $ M := [A \;\; B] $ 
are unknown, the optimal policy is usually not directly attainable.
Suppose we allow an agent to interact with the system for a certain amount of time. 
One common measure of the agent's performance is \emph{regret}: the difference between the total cost that would be incurred by the unknown optimal policy and that by the agent. 
A good agent would achieve a regret upper bounded by a sub-linear function of the amount of time she is allowed to play. In this case, the \emph{average regret per unit time} tends to zero, and hence the agent's performance becomes closer to an unknown optimal policy when she is allowed to interact with system longer.
The amount of time that an agent takes to obtain a small (e.g. constant) average regret is called the \emph{sample complexity of learning}.

There are extensive studies on learning to control an unknown LQR system.
As we will discuss in related works, prior works \citep{abbasi2011regret, ibrahimi2012efficient, ouyang2017control} achieve (at best) $\widetilde{\mathcal{O} } \left( \mathrm{poly}(d) \sqrt{T} \right)$ regret rate, despite possibly stronger assumptions. 
The algorithms proposed in these papers all achieve sublinear regret bound. 
However, their sample complexities are at least proportional to the total number of parameters in the model $M = [A \;\; B]$ (i.e., polynomial in $d$). 
This complexity, however, can be big in practice.
For instance, in video games, the states -- video frames -- have a huge number of pixels. Problems as classic as the \texttt{Mountain-Car} and \texttt{Cart-Pole} \citep{1606.01540} can also suffer from this problem when learned with images -- the images representing the states have thousands of pixels. 
Nevertheless, we notice that \textbf{(1)} in all these examples, the intrinsic dimension (i.e., the internal state space is about $3$-$6$ dimensions) is actually low, and \textbf{(2)} \citet{suh2020surprising} recently shows linear models with visual feedbacks achieve surprisingly good control. 
We therefore propose to use LQR for low-rank problems and ask the following question. 

{\centering
\emph{Is there an online algorithm that learns to control an unknown LQR system with number of samples only depends on the intrinsic model complexity?}
}

To make this precise, we consider a basic low intrinsic complexity setting: the states have an underlying \emph{low-rank representation} (please refer to Definition \ref{def:rank-m-controllable} for details). 
Formally, we consider the episodic online LQR control problem.
In each episode of such problems, the agent starts from a random/adversarial initial state, and execute $H$ control actions to finish. After $H$ steps, the agent starts over from another initial state, and the next episode begins. 
Our goal is to minimize the number of episodes for the agent to achieve a constant average regret (per episode).
In this paper, we answer the above question by proposing an algorithm that obtains a $K$-eposide regret bound of order
\[
\widetilde{\mathcal{O}} \left( m^{3/2} \sqrt{K} \right) \footnote{$\tilde{\mathcal{O}}$ omits factors that are $poly$-$log$ in the inputs, as well as factors depending  on $H$. A formal statement can be found in Theorem \ref{thm}.},
\]
where $m$ bounds the rank of the matrix $M = [A \quad B]$. Our algorithm corresponds to a 
sample complexity of $\mathrm{poly}(m)$. 
This order can be significantly smaller than previous results, which depend polynomially on the ambient dimension $d$.

Our result is a technically involved combination of the Optimism in the Face of Uncertainty (OFU) principle \citep{dani2008stochastic, abbasi2011regret, lale2019stochastic, kveton2017stochastic}, and low-rank approximations, e.g., principle component analysis (PCA) \citep{jolliffe1986principal, vaswani2017finite}. 
On a high level, our algorithm can be viewed as a model-based algorithm for closed-loop control. 
We estimate the system dynamics using least-squares combined with a PCA projection. 
In each episode, with the learned model and its uncertainty estimation, we compute an optimistic control policy that carefully balances exploration and exploitation. 
We then execute the control policy to obtain a new episode of data and at the same time incur provably small regret. 

In our analysis, the core technical difficulty comes from the interleaving of PCA and LQR transitions: for each episode, the new data points can potentially lie outside the subspace identified by the PCA projection of previous episodes.
Therefore a plain adoption of previous results fail to give a bound that is independent of $d$. 
In order to handle this issue, we project all data points to the subspace identified by the PCA at the very last episode, and carefully handle the difference between the PCA subspaces across episodes. 
By modeling {the process} as a rank-deficient self-normalized random process, we show that our algorithm provably achieves a sublinear regret that only depends on the internal rank of the system. 

\textbf{Related works. }
The history of control theory can date back to the study of governors by Maxwell \citep{maxwell1868governors}, where he linearized differential equations of motion. This work, together with the classic Riccati equation \citep{ricaati1720, bittanti1996history}, builds the root foundation of modern LQR. 
Similar to many control problems, LQR problems can be classified into open-loop problems \citep{ljung1983theory, helmicki1991control, chen1993caratheodory, box2015time, hardt2018gradient, tu2017non, dean2017sample} versus closed-loop problems. 
Compared to open-loop problems, closed-loop problems are closer to a reinforcement-learning setting -- feedbacks of the environment are used in an interactive fashion. In this paper, we focus our attention to the closed-loop LQR problem, and use LQR problems to refer to closed-loop LQR problems from now on. 

In recent years, as motivated by an increasing amount of real-world data-driven applications, more interests are attracted to the learning to control problems -- control problems where the system dynamics are unknown. The learning to control problem is also known as system identification (for observable systems) in classic terms \citep{kalman1960new}, and learning based model predictive control methods have been developed by the control community \citep[e.g.,][]{aswani2013provably,koller2018learning}. 
Among works on learning to control for LQR problems, some sit in a ``bandit'' setting, i.e. one observation right after one action, and no rollouts are allowed. \citet{abbasi2011regret} use the optimistic principle, and obtained a regret bound of order $\widetilde{\mathcal{O}} \left( f(d) \sqrt{T} \right)$, where $f(d) $ could be exponential in the ambient dimension $d$. \citet{ibrahimi2012efficient} makes a sparsity assumption on the system dynamics $M = [A \quad B]$ and achieves a regret bound of order $\widetilde{\mathcal{O}} \left( d \sqrt{T} \right)$. Yet the dependence on the ambient dimension $d$ is not removed. \citet{simchowitz2020naive} proposes to use $\epsilon$-greedy exploration and achieves optimal rate in terms of the ambient dimension. For observable systems, online control with system identification for Linear Quadratic Gaussian models have also been studied \citep{lale2020lqg}, and a regret depends polynomially on the ambient dimension is derived. Assuming a correct specification of the prior distribution, Bayesian methods have also been applied to learning to control LQRs \citep{abeille2017thompson, ouyang2017control, abeille2018improved}. However, all the above mentioned algorithms admit regret (at best) of order $\widetilde{\mathcal{O}} \left( poly(d) \sqrt{T} \right)$. 


LQR problems has also been studied in a ``non-bandit'' setting, i.e. rollouts of trajectories are permitted. 
In this setting, efforts on extracting the intrinsic dimension have been made. 
The concepts of ``Bellman rank'' \citep{jiang2017contextual} and ``witness rank'' \citep{sun2019model} are proposed as dimension measures in this  ``non-bandit'' setting. These methods, however, are not as sample efficient as ours due to the need of rollouts.






 



\section{Preliminaries and Notations}
In an episodic LQR problem, the agent starts the $k$-th episode from a random initial state sampled from $\mu$: $x_{k,1} \sim \rho$. The agent then executes $H-1$ controls to finish this episode. Episode $k$ ends at $h = H$, and the agent starts over from $h = 1$and the $(k+1)$-th episode starts, where a new initial state $x_{k+1,1}$ is sampled from $\rho$. 
At each step $(k,h)$, the next state of the system depends linearly on the current observed state $\hat{x}_{k,h} \in \mathcal{X} \subseteq \mathbb{R}^d$ and the action taken $u_{k,h} \in \mathcal{X} \subseteq \mathbb{R}^{d_u}$ plus a noise term. In other words, there are matrices $ A \in \mathbb{R}^{d \times d}$, $B \in \mathbb{R}^{d \times d_u } $, such that the next state can be described as 
\begin{align*}
    \hat{x}_{k, h+1} &= A \hat{x}_{k, h} + B u_{k,h} +
    w_{k,h}, 
\end{align*}
where $w_{k,h} \in \mathbb{R}^d$ is a mean-zero noise, $\hat{x}_{k, h+1}$ is the observed state. 
We write $M = [A, B]$ and $\hat{z}_{k,h} = [ \hat{x}_{k,h}^\top, u_{k,h}^\top ]^\top$. The transition can then be rewritten as $\hat{x}_{k,h+1} = M \hat{z}_{k,h} + w_{k,h}$. 

At each time $(k,h)$, the system transits from $\hat{x}_{k,h}$ to $\hat{x}_{k,h+1}$ and receives an immediate cost 
\begin{align}
    c_{k,h} = 
        \hat{x}_{k,h}^\top Q_h \hat{x}_{k,h} + u_{k,h}^\top R_h u_{k,h}, 
     \label{eq:def-cost-c}
\end{align}
where $Q_h$ ($h \in [H]$) and $R_h$ ($h \in [H-1]$) are known positive definite matrices, and $R_H = 0$. 

The goal of the agent is to learn a policy $\pi:[H]\times\mathbb{R}^d\rightarrow \mathbb{R}^{d'}$, such that the following objective is minimized for all $h\in[H]$
\begin{align}
    J_{k,h}^\pi (M , x) := \mathbb{E}_{\pi} \left[ \left. \sum_{h^\prime = h}^H c_{k,h} \right| \hat{x}_{k,h} = x \right], \label{eq:def-J-via-c}
\end{align}
where $c_{k,h}$ is the immediate cost per step and $\mathbb{E}_{\pi}$ is over the random trajectory generated by policy $\pi$ starting from $x$ at $(k,h)$.
When it is clear from context, we omit $k$ from $J_{k,h}^\pi$ and write $J_{k,h}^\pi$ as $J_{h}^\pi$.

%
%

From Bellman optimality \citep{bellman1954theory, bertsekas2004dynamic}, for a system with dynamics $M = [A\quad B]$, the optimal policy $\pi^*$ is given by \citep[e.g.,][]{bertsekas2004dynamic}: $\forall x \in \mathcal{X}, h \in [H-1], $
%
%
%
\begin{align}
    \pi_h^* (x) &:= \mathcal{K}_h (M ) x, \quad \text{ where } \nonumber \\ 
    \mathcal{K}_h (M) &:= - \left( R_h + B ^\top \Psi_{h+1} (M)  B \right)^{-1} B^\top \Psi_{h+1} (M)  A,\label{eq:control-law} 
\end{align}
and $\Psi_h (M)$ is defined by the Riccati iteration: 
\begin{align}
    &\Psi_H(M)  := Q_H , \nonumber \\
    &\Psi_{h-1}(M) := Q_h \hspace{-2pt} + \hspace{-2pt} A^\top \Psi_{h} (M) A \hspace{-2pt} \nonumber \\
    &\quad - A ^\top \Psi_{h} (M) B \left( R_h \hspace{-3pt} + \hspace{-3pt} B^\top \Psi_{h} (M) B \right)^{-1} \hspace{-3pt} B^\top \hspace{-2pt} \Psi_{h} (M) A,  \label{eq:def-Ph} 
\end{align} 
for $h < H.$ 
When it is clear from context, we simply write $\Psi_h$ for $\Psi_h (M)$. 
As we will later discuss in Section \ref{sec:well-defined}, the matrix $K_h(M)$ is always well-defined, since the matrix $ R_h + B ^\top \Psi_{h+1} B $ is invertible for arbitrary $M = [A,B]$ (e.g., \cite{bertsekas2004dynamic}). 
More details regarding well-definedness of the control are in Appendix \ref{app:well-defined}. This property allows us to estimate the the dynamics $M$ and design a control policy based on the estimation. For a system with dynamics $M = [A,B]$, the cost under an optimal policy can be written as follows (e.g., p. 229, Chapter 5, Vol. I, \cite{bertsekas2004dynamic}),
\begin{align}
    &J_{h}^* (M , x) = x^\top \left[ \Psi_{h+1} (M) \right]  x + \psi_h , \label{eq:J-from-psi} \\
    &\psi_h \vspace*{-2pt} := \psi_{h+1}  + \mathbb{E}_{w_{h+1}} \bigg[ w_{h+1}^\top \left[ \Psi_{h+1} (M) \right] w_{h+1} \bigg] \vspace*{-2pt}, \; \psi_H := 0. \nonumber
\end{align}  

\textbf{Controllable Subspace.} 
In control theory, it is often assumed that the linear quadratic system is controllable, i.e., the matrix 
\begin{align*}
    \left[ B \quad A B \quad  A^2 B \cdots  A ^{d - 1} B  \right] 
\end{align*}
is of full row rank. 
When the above condition is satisfied, we say that $[A \quad  B ]$ is a pair of \textit{controllable matrices}. Intuitively, a pair of controllable matrices allows the agent's action (control) to influence all dimensions of the system. 
In our problem, we assume a low-rank version of controllability. 

\begin{definition} \label{def:rank-m-controllable}
	We say that a pair of matrices $\left[ A \quad B \right]$ is rank-$m$ controllable if there exists a matrix $L$ of $m$ orthonormal columns, such that such that $A = LL^\top A LL^\top $, $B = LL^\top B$ and $ \left[ L^\top A L \quad L^\top B \right]$ is a pair of controllable matrices. The projection matrix $P = LL^\top$ is called the true \textbf{projection matrix for system $M = [A \quad B]$}. 
\end{definition}

Intuitively, a pair of rank-$m$ matrices defines a transition dynamic such that, if there were no noise, the states would always lie on a rank-$m$ subspace, and when restricted to this rank-$m$ subspace, the system is controllable. When $m$ equals the dimension of the state space, rank-$m$ controllability is equivalent to controllability. Throughout the rest of the paper, we assume that the system we are considering is rank-$m$ controllable for a fixed $m\le d$. We focus on settings where $m \ll d$ and $m = \Theta \left( d_u \right)$, where $d_u$ is the dimension of the action (control) space. 


In a finite-horizon discrete-time setting, for positive definite $R_h$ and $Q_h$, the optimal control law can be solved via dynamic programming. \citep[See e.g., p. 150, Chapter 4, Vol. I; p. 229, Chapter 5, Vol. I, ][]{bertsekas2004dynamic}. We assume rank-$m$ controllability so that the controls can influence the entire subspace on which the noiseless states lie. Problems such as the connection between rank-$m$ controllability and convergence of the Riccati iteration might be an interesting future direction. 

    

\textbf{Performance Measure. }
We use regret to measure the performance of the algorithm. For LQR problems in this paper, if the true system dynamics are $M_* = [A_* \quad B_*] $, the regret of the first $K$ episodes is defined as:
\begin{align}
    Reg (K) = \sum_{k=1}^K J_1^{\pi_k} ( M_*, \hat{x}_{k,1} ) - J_1^{*} ( M_*, \hat{x}_{k,1} ), \label{eq:def-regret}
\end{align}
where $\hat{x}_{k,1}$ is the starting state for episode $k$, $J_1^{\pi_k}$ is computed from (\ref{eq:def-J-via-c}), and $J_1^* (M_*, \hat{x}_{k,1})$ can be calculated from (\ref{eq:J-from-psi}).  
As discussed above, $J_1^{*} ( M_*, x_{k,1} )$ is the (expected) cost of an optimal policy for episode $k$, and $J_1^{\pi_k} ( M_*, x_{k,1} )$ is the (expected) cost of the policy executed during episode $k$. As common in bandit and online learning setting, we want a sub-linearly growing regret. This ensures that the strategy incurs optimal costs given enough time. 

\section{Learning with Low-rank Structure} 
In learning settings, the system dynamics are unknown, and the task is to learn a good policy as we interact with the environment. 
We will apply the optimism in the face of uncertainty (OFU) principle to learn a good policy.
In particular, we maintain an estimation of the system dynamics as well as its uncertainty.
To control the system, we will be ``optimistic'' in the uncertainty ball of our model estimator, i.e., we will use the best possible (in terms of cost) model that satisfies our uncertainty estimation to solve for the next policy. 
In order to obtain such ``optimistic'' estimations, we (i) apply a PCA projection to the observed data and use a least-square regression to fit the system dynamics; and (ii) we search a ``confidence region'' close to this estimated system dynamics. This confidence region uses the uncertainty in both the regression and the PCA projection. 

Throughout our analysis, 
we use the following common stability assumption. 

\begin{assumption}[Stability]
\label{assumption:bounds}
    Let $M_* = [A_* \quad B_*] $ be the true system dynamics. We assume that~\footnote{For simplicity, we use $C$ to denote most boundedness constants. The only exception is the noise bound $C_w$, and consequently the constant $C_{\max}$ as per defined in Algorithm \ref{alg}. This does not lose any generality.} 
    \newline 
    \textbf{(1)}For all $h \in [H]$, $\| A_* + B_* \mathcal{K}_h (M_*) \|_2 \le r$ for some $r < 1$ and $\| \mathcal{K}_h (M_*) \|_2 \le C$ for some constant $C$, where $\mathcal{K}_h (M_*)$ is defined in (\ref{eq:control-law}). We assume that $ \| M_* \| \le C $ for some constant $C$. 
    \newline
    \textbf{(2)} The control mapping $\mathcal{K}_h$ is Lipschitz near $M_*$: there exist constants $C$ and $ D $ such that $ \| \mathcal{K}_h \( M_* \)- \mathcal{K}_h \( M \) \|_2 \le D \| M_* - M \|_2 $ for all $M $ such that $ \| M_* - M \|_2 \le C $. 
    \newline 
    \textbf{(3)} The noises $w_{k,h}$ satisfy: $\forall k \ge 1, h \in [H]$, $\forall h \in [1,H]$, $\left\| {w}_{k,h} \right\|_2 \le C_w < 1$ and that $2 C_w + r \le 1$. 
    \newline
    \textbf{(4)} The initial states for each episode is bounded: for any $k\ge 1$, $\left\| \hat{x}_{k,1} \right\|_2 \le 1$. 
\end{assumption} 

In Assumption \ref{assumption:bounds},  item (1) in  is a type of stability assumption. Stability is standard and usually assumed for control problems \citep{ibrahimi2012efficient, matni2017scalable, dean2019sample}. Item (2) is actually naturally true, since all mappings are continuous, and thus Lipschitz continuous on a compact region. Items (3) and (4) are assumed for simplicity, since we can always rescale the space so that these two are satisfied. 

We also assume rank-$m$ controllability of the problem. 

\begin{assumption} \label{assumption:low-dimension}
    We assume that the system dynamics matrices $ [A_* \quad B_*] $ are rank-$m$ controllable (Definition \ref{def:rank-m-controllable}). We use $P_*$ to denote the true projection matrix for this system (Definition \ref{def:rank-m-controllable}). 
\end{assumption} 
The above low-rank assumption distinguishes our problem from a general LQR problem. Utilizing the low-rankness can improve the regret dependence on dimension significantly. 
We also make a noise assumption. 

\begin{assumption}[Noise] \label{assumption:noise}
	We assume that at any $h \in [H]$ and $k \ge 1$, the noise $w_{k,h}$ is (1) independent of all other randomness, (2) $\mathbb{E} [w_{k,h}] = 0$ for any $h \in [H]$, and $k \ge 1$. (3*) Without loss of generality, there exists constant $\sigma^2$, such that $\mathbb{E} [ w_{k,h} w_{k,h}^\top ] = \sigma^2 I_d$ for all $k$ and $h $. 
\end{assumption}
Items (3*) can be relaxed by combining Remark 3 in \citep{abbasi2011regret} and PCA analysis with general noises \citep{vaswani2017finite}. 
We focus on this standard noise setting for a cleaner presentation, while our results generalize to different noise settings.

For representation simplicity, we introduce the following assumption.
\begin{assumption}[Initial Distribution] \label{assumption:state-eigenvalue}
    We assume there exists $\lambda_- >0$, such that the unseen starting state $x_{k,1}$ satisfies 
	\begin{align*}
	    \lambda_m \left( \mathbb{E}_\rho \left[ {x}_{k,1} \left( {x}_{k,1} \right)^\top  \right] \right) \ge \lambda_-,
	\end{align*}
    where $\mathbb{E}_\rho $ is the expectation with respect to the initial distribution $\rho$, and $\lambda_m ( \cdot )$ returns the $m$-th eigenvalue of a matrix. 
\end{assumption} 

Note that this is a mild assumption on certain exploratory property of the initial distribution. If we do not have such an initial distribution, we can use a fraction of steps to maximize the top-$m$ eigenvalues. Maximizing the top $m$ eigenvalues can be done by simply keep playing a same control $u$ (because of low-rank controllability). Our analysis techniques still carry through.

Before formulating our algorithm, we define the following notations. 

\vspace{-0.3cm}
\begin{itemize}[align=left,leftmargin=*]
	\item For state vector $\hat{x}_{k,h}$ and controls vector $u_{k,h}$ at $(k,h)$, we write $\hat{z}_{k,h} 
    	:= 
    	\begin{bmatrix}
    	\hat{x}_{k,h} \\
    	{u}_{k,h}
    	\end{bmatrix}$. 
    \item For any $k \ge 1$, we define for following matrices 
    \begin{align}
    &\hat{X}_{k} := 
	\left[ \hat{x}_{h', k'} \right]_{ \substack{ h' \in [ H-1],  k' \in [k-1] } }, \\
	&\hat{X}_{k}^{\mathrm{next}} := 
    \left[ \hat{x}_{h', k'} \right]_{ \substack{ h' \in [2, H] ,  k' \in [ k-1] }}, \label{eq:notation1} \\
	&U_{k} := 
    \left[ u_{h', k'} \right]_{ \substack{ h' \in [H-1] ,  k' \in [ k-1] } }, \\
	&W_{k} := 
	\left[ w_{h', k'} \right]_{ \substack{ h' \in [ H-1],   k' \in [ k-1] } } , \; 
\\
     &\qquad  \hat{Z}_k 
    	:= \left[ \hat{X}_{k}^\top \quad U_{k}^\top  \right]^\top 
    	\label{eq:notation2}
	\end{align} 
\end{itemize}
In the above, $\hat{X}_k$ is the collection of observed states ($h$ runs from $1$ to $H-1$): Each column in $\hat{X}_k$ is an observed state. 
Similarly, $\hat{X}_k^{\mathrm{next}}$ is the collection of observed ``next'' states ($h$ runs from $2$ to $H$). $U_k$, $ W_k $ and $\hat{Z}_k$ are controls, noises (not directly observable), and state-control pairs respectively. 
Using the above notations, we can write 
\begin{align*}
    \hat{X}_k^{\mathrm{next}} &= M_* \hat{Z}_k + W_k,  \quad \text{or} \\
    \hat{X}_k^{\mathrm{next}} &= A_* \hat{X}_k + B_* U_k + W_k .
\end{align*}
With the above notations introduced, we can proceed to design our control rules. 
\subsection{Online Control Design}
With the above notations (Eq.~\ref{eq:notation1},~\ref{eq:notation2}) and true dynamics matrix $M_*$, we have $\hat{X}_k^{\mathrm{next}} = M_* \hat{Z}_k + W_k $. Thus to approximate $M_*$, we can consider finding a matrix $M$ for the following problem  
\begin{align}
    \min_M \left\| M \hat{Z}_k - \hat{X}_k^{\mathrm{next}} \right\|_F^2 + \frac{1}{2} \left\| M \right\|_F^2. \label{eq:regression-problem}
\end{align}
Since the solution to this matrix ridge regression problem is $ \left( \hat{Z}_k \hat{Z}_k^\top + I_{d + d_u} \right)^{-1} \hat{Z}_k \left( X_k^{\mathrm{next}} \right)^\top $, we combine a PCA with this solution, and design a rank-$m$ estimate. 
To compute this PCA, we apply a singular value decomposition to $\hat{X}_{k-1}$. Let $L_k $ be the matrix of left singular vectors of $\hat{X}_{k-1}$. Let $ \bar{L}_k $ be the columns of $L_k$ that corresponds to the top $m$ singular values.
     The learned projections at episode $k$ are
     \begin{align}
     	P_k := \bar{L}_k \bar{L}_k^\top  \quad \text{and} \quad P_k^{\mathrm{\mathrm{aug}}} := 
     	 \begin{bmatrix} 
    		{P}_{k} & 0_{d \times d_u} \\ 0_{d_u \times d } & I_{d_u } 
    	\end{bmatrix} . 
     	\label{eq:learn-projection}
     \end{align} 
The projection $P_k$ is for states $\hat{x}_{k,h}$, and $P_k^{\mathrm{aug}}$ is multiplied to state-action pairs $\hat{z}_{k,h}$. 
More specifically, $ P_k $ applies to $\hat{x}_{k,h}$ and projects the states to a rank-$m$ subspace. $P_k^{\mathrm{aug}}$ applies to $\hat{z}_{k,h} = \begin{bmatrix} \hat{x}_{k,h} \\ u_{k,h} \end{bmatrix} $, projects the states $\hat{x}_{k,h}$ to a rank-$m$ subspace and preserves the controls $ u_{k,h} $. 
With the learned projections $P_k $ and $P_k^{\text{aug}}$, we compute the following quantities: 
\begin{align} 
	&\widetilde{V}_k := \left( \hat{Z}_k \hat{Z}_k^\top + I_{d+d_u } \right), \quad {V}_k := P_k^{\mathrm{aug}} \widetilde{V}_k P_k^{\mathrm{aug}},  \nonumber \\
	&\bar{Z}_k := P_k^{\mathrm{aug}} \hat{Z}_k, \quad \bar{X}_k^{\mathrm{next}} := P_k \hat{X}_k^{\mathrm{next}}.  \label{eq:def-Vk}
\end{align} 
With the above quantities, we can estimate the system dynamics $M_*$ by 
\begin{align}
    M_k^\top = V_k^\dagger \bar{Z}_k \left( \bar{X}_k^{\mathrm{next}} \right)^\top, \label{eq:compute-Mk} 
\end{align}
where ${}^\dagger$ is pseudo-inverse operator. Intuitively, (\ref{eq:compute-Mk}) is a low-rank approximation to the solution of the problem in (\ref{eq:regression-problem}). 

Now we define confidence region around $M_k$. Fix a parameter $\delta$ that controls the probability that the regret behaves nicely. 
After a warmup period $K_{\min}$, the confidence region $\mathcal{C}^{(k)}$ ($k > K_{\min}$) is defined as $ \mathcal{C}^{(k)} := \mathcal{C}_* \cap \mathcal{C}_1^{(k)} \cap \mathcal{C}_2^{(k)} $, where $\mathcal{C}_*$ is a fixed closed set that always contains $M_*$, and 
\begin{align}
    &\mathcal{C}_1^{(k)} := 
    	\left\{ M : 
    		\left\| 
        	 (M - M_k) (I_{d + d_u }  -  {P}_k^{\mathrm{aug}}) 
        	\right\|_2 \lesssim  
        	 G_{k,\delta} 
    	\right\}, \nonumber \\  
    	&\qquad \text{ with } G_{k,\delta} = \widetilde{\Theta} \left( \frac{1}{\sqrt{k}} \right) ,  \label{eq:con-proj}  \\
    &\mathcal{C}_2^{(k)} := \left\{ M  : \left\| ( M - M_k ) V_k^{1/2}  \right\|_2^2 \lesssim \beta_{k,\delta}  \right\}, \nonumber \\
    &\qquad \text{ with } \beta_{k,\delta} = \widetilde{\Theta} \left( 1 \right) 
    \label{eq:con-noise} .   
\end{align}
For practitioners, the values of $G_{k,\delta}$ and $\beta_{k,\delta}$ can be rescaled by proper constants, and $\mathcal{C}_*$ can be properly chosen to regularize the norm of learned transitions. For theoretical purpose, we use
$$ \mathcal{C}_* := \left\{ M = [A \quad B]:  \| A + B \mathcal{K}_h (M) \|_2 \le r \text{ and } \| M \| \le C \right\}, $$ and 
use 
$ G_{k,\delta} $ and $ \beta_{k,\delta} $ as detailed in Appendix \ref{app:notation}. Intuitively, $ \mathcal{C}_1^{(k)} $ defines a region (around $M_k$) that is perpendicular to the PCA projection, and $\mathcal{C}_2^{(k)}$ defines a region around $M_k$ parallel to the PCA projection. 
As we will show later, with high probability, $ M_* \in \mathcal{C}^{(k)} = \mathcal{C}_* \cap \mathcal{C}_1^{(k)} \cap \mathcal{C}_2^{(k)} $. 
We will search within $\mathcal{C}^{(k)}$ for an optimistic estimate $ \widetilde{M}_k $. Specifically, within the confidence region $\mathcal{C}^{(k)}$, we find an optimistic estimation $\widetilde{M}_k \in \mathcal{C}^{(k)}$, such that 
\begin{align}
    \widetilde{M}_k  \in \arg \min_{M \in \mathcal{C}^{(k)} } J_1^* (M, \hat{x}_{k,1}), \label{eq:compute-tilde-Mk}
\end{align} 
where $ J_1^* (M, \hat{x}_{k,1}) $ is the optimal cost if the system transition were $M$, and can be computed using (\ref{eq:J-from-psi}). 
With this estimation $\widetilde{M}_k$, we play a policy $\pi^{(k)} = \left\{ \pi_h^{(k)} \right\}_{h \in [H]}$:
\begin{align}
    \pi_h^{(k)} (x) := \mathcal{K}_h ( \widetilde{M}_k ) x.  \label{eq:control-law-learning}
\end{align}
where $\mathcal{K}_h ( \widetilde{M}_k )$ is defined in (\ref{eq:control-law}). Our strategy is summarized in Algorithm \ref{alg}.

\begin{algorithm*}[t!]
    \caption{Low-rank LQR with OFU} 
    \label{alg}
    \begin{algorithmic}[1] 
       	\STATE \textbf{Parameters:} \textbf{\textit{horizon}} $H > 0$, \textbf{\textit{probability parameter}} $\delta$, \textbf{\textit{dimension}} $d$, \textbf{\textit{true rank}} $m$, \textbf{\textit{constant bounds}} $C$ and $C_{\max} := 4 C_w + 2 \sqrt{2} C_w^2$, 
       	\textbf{\textit{minimal eigenvalue}} $\lambda_-$ (Assumption \ref{assumption:state-eigenvalue}),
       	    \textbf{\textit{total number of episodes}} $K$, \textbf{\textit{warm-up period}} $K_{\min} = \Kmin$.
       	    
       	\STATE $\qquad \rhd $ \textit{In practice, the above parameters can be chosen more freely. }
       	
       	\STATE
       	\STATE \textbf{Warmup:} For the first $ K_{\min} $ episodes, randomly play controls from a bounded set. 
        \FOR{$k = K_{\min} + 1, K_{\min} + 2, \cdots, K $}
            \STATE Observe $\hat{x}_{k,1}$. 
            \STATE Compute $\widetilde{M}_k := \arg \min_{M \in \mathcal{C}^{(k)} } J_1^* (M, \hat{x}_{k,1}) $, where $ J_1^* (M, \hat{x}_{k,1}) $ is computed from (\ref{eq:J-from-psi}). 
        	\FOR{$h = 1,2, \cdots, H-1$}
        		\STATE Execute the control $u_{k,h} = K_h ( \widetilde{M}_k ) \hat{x}_{k,h}$, where $K_h ( \widetilde{M}_k ) $ is defined in (\ref{eq:control-law}). 
        		\STATE Observe the next state $\hat{x}_{k,h+1}$. 
        	\ENDFOR
        	\STATE Gather the observations into $\hat{Z}_k$, $\hat{X}_k$, and $\hat{X}_k^{\mathrm{next}}$. 
        \ENDFOR
    \end{algorithmic}
\end{algorithm*}

\subsection{Well-definedness of the Control}
\label{sec:well-defined}
Before analyzing the algorithm performance, we first need to show that the algorithm is well-defined, i.e., (\ref{eq:control-law-learning}) is well-defined. 
To show this, we need 

(i) $\widetilde{M}_k$ is well-defined (with high probability); \\
(ii) Given any $ \widetilde{M}_k $, the matrix $\mathcal{K}_h (\widetilde{M}_k)$ (Eq. \ref{eq:control-law}) is well-defined. 

For item (i), 
in Proposition \ref{prop:bounded-Mk}, we show that $\mathcal{C}^{(k)}$ is closed, bounded and non-empty (with high probability), which shows we can find  $\widetilde{M}_k \in \arg\min_{M \in \mathcal{C}^{(k)}} J_1^* (M, \hat{x}_{k,1} ) $ with high probability. 
\begin{proposition} \label{prop:bounded-Mk}
	The regions enclosed by $\mathcal{C}^{(k)}$ ($k \in ( K_{\min}, K]$) are closed and bounded. 
    Also, under event $\mathcal{E}_{K,\delta}$,   $\mathcal{C}^{(k)}$ is non-empty. 
\end{proposition} 

Proposition \ref{prop:bounded-Mk} is essentially a boundedness/stability result. This ensures that the controls $u_{k,h}$ are of reasonable length. 
A proof of Proposition \ref{prop:bounded-Mk} can be found in Appendix \ref{app:well-defined}. 
 
For item (ii), given any $ M = [A\quad B]$, which defines a transition system, the matrix $K_h (M)$ (Eq. \ref{eq:control-law}) is well-defined for all $h \in [H]$. Since $ \mathcal{K}_h (M) := - \left( R_h + B ^\top \Psi_{h+1}  B \right) ^{-1}   B^\top \Psi_{h+1}  A $, it is sufficient to show $\Psi_h $ (defined in Eq. \ref{eq:def-Ph}) is positive semi-definite. This is because positive semi-definiteness of $\Psi_{h+1}$, together with positive definiteness of $R_h$, ensures that $R_h + B ^\top \Psi_{h+1}  B$ is invertible. 

The positive semi-definiteness of $\Psi_h$ is below in Lemma \ref{lem:psd}. Its proof can be found in textbooks covering Linear Quadratic Regulators (e.g., \cite{bertsekas2004dynamic}). We provide a proof in Appendix \ref{app:well-defined} for completeness.  

\begin{lemma} \label{lem:psd} 
    The matrix $\Psi_h $ is positive semi-definite for any $h \in [H]$ and $M$, provided that $Q_h , R_h $ are positive definite. 
\end{lemma} 
\section{Regret Analysis} 


The regret can be bounded as in Theorem \ref{thm}. 

\begin{theorem} \label{thm}
	Under Assumptions \ref{assumption:bounds}-\ref{assumption:state-eigenvalue}, for any $\delta > 0$, with probability at least $1 - (4 K + 2) \delta$, the regret for the first $K$ ($K > K_{\min }^2$, $K_{\min} := \Kmin$) episodes satisfies 
	\begin{align*}
		&\mathrm{Reg} (K) \le \mathcal{O} \left( \left( H^{5/2} + m^{3/2} H \right) \sqrt{K} \hspace{3pt}  \mathrm{polylog} \left( \frac{K}{\delta} \right)  \right), 
	\end{align*}
	where $\mathcal{O}$ omits (poly)-logarithmic terms in $H$ and $d$. 
	
\end{theorem}


To bound the regret, we first show that the event $\mathcal{E}_{K,\delta} : = \left\{ M_* \in \mathcal{C}^{(k)}, \quad \forall \text{  } k \in ( K_{\min}, K] \right\}$ holds with high probability (Section \ref{sec:part1}). Then we bound the regret under event $\mathcal{E}_{K, \delta}$ 
(Section \ref{sec:part2}).

%
\subsection{Part I: $ M_* \in \mathcal{C}_1^{(k)} \cap \mathcal{C}_2^{(k)} $ for all $k = K_{\min}+1, K_{\min}+2, \cdots, K$ with high probability} \label{sec:part1}


\textbf{Part Ia: $M_* \in \mathcal{C}_1^{(k)}$ with high probability.} To show $M_* \in \mathcal{C}_1^{(k)}$ with high probability, we need to show that the projection error is small. In other words, we need to bound the term $\left\| P_* - P_k \right\|_2$. The tool we will use is Lemma \ref{lem:projection}, 
which extends previous results (Corollary 2.7  by \cite{vaswani2017finite}, Theorem 1 by \cite{lale2019stochastic}) to our case. 

\begin{lemma}\label{lem:projection} 
	Let $P_*$ be the true projection matrix for the rank-$m$ controllable system $M_* = [A_* \quad B_*]$. Let $C^{\max} := 4 C_w + 2 \sqrt{2} C_w^2$. Suppose Assumptions~\ref{assumption:bounds}-\ref{assumption:state-eigenvalue} hold. For $k > K_{\min } = \Kmin$, 
	with probability at least $1 - 3 \delta$, 
	\begin{align*}
		\left\| P_* - P_k \right\|_2	  \le G_{ k, \delta}, \quad \text{where $  G_{ k, \delta} := \widetilde{\Theta} \left( \frac{1}{\sqrt{k}} \right)$ }.
	\end{align*}
\end{lemma}

The proof of Lemma \ref{lem:projection} uses the Davis-Kahan theorem on principle angle between column spans, as well as concentration results for matrix martingales. More details of this proof can be found in Appendix \ref{app:projection}. 


\textbf{Part Ib: $M_* \in \mathcal{C}_2^{(k)}$ with high probability. } 
To show $M_* \in \mathcal{C}_2^{(k)}$ with high probability, we need to bound the quantity $\left\| ( M_* - M_k ) V_k^{1/2}  \right\|_2^2$. By definition of  $V_k$ (in Eq. \ref{eq:def-Vk}), the norm of $V_k$ increases with $k$. At roughly the same rate, the residual $\| M_* - M_k \|_2$ decreases with $k$ because of the learning nature. Using this observation, we can get that, with high probability, $\left\| ( M_* - M_k ) V_k^{1/2}  \right\|_2^2 \le \widetilde{\Theta} \( 1 \)$. The full proof requires a rank-deficit self-normalized process formulation, whose details are in Appendix  \ref{app:rank-deficit}. 

Combining Part Ia and Part Ib immediately gives Lemma \ref{lem:high-prob}. 
\begin{lemma} \label{lem:high-prob}
	With probability at least $1 -  4 K \delta$, for $K > K_{\min} = \Kmin$, the following event is true: 
	\begin{align}
	    \mathcal{E}_{K,\delta} := \left\{ M_* \in \mathcal{C}^{(k)}, \; \forall k = K_{\min} + 1,  \cdots, K \right\}.
	\end{align}
\end{lemma} 

Some more details on Lemma \ref{lem:high-prob} can be found in Appendix \ref{app:pf-high-prob}. In the next part, we will bound the regret under event $ \mathcal{E}_{K,\delta} $. 



\begin{figure*}[t] 
	\centering
	\includegraphics[width = 0.42\textwidth]{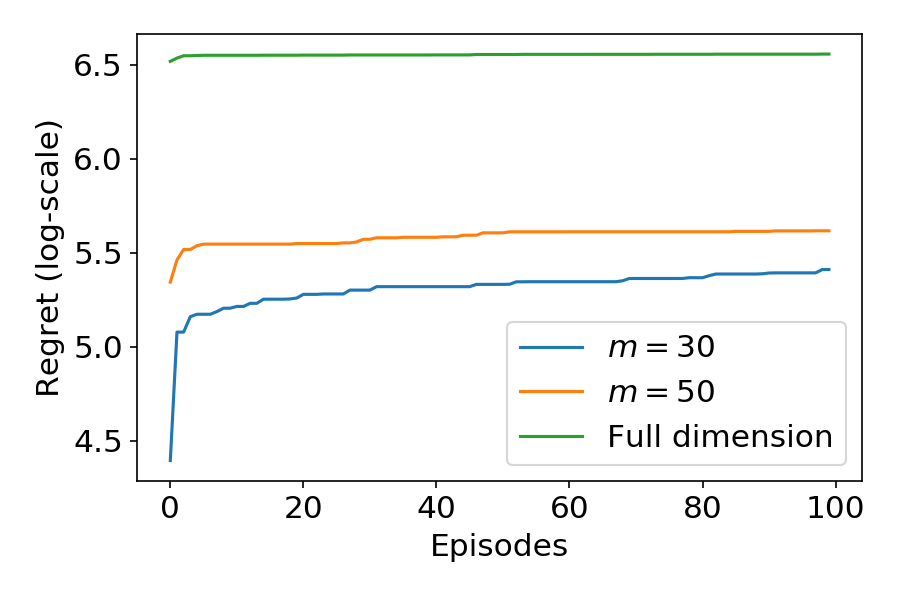} 
	\hspace{10pt} 
    \includegraphics[width = 0.42\textwidth]{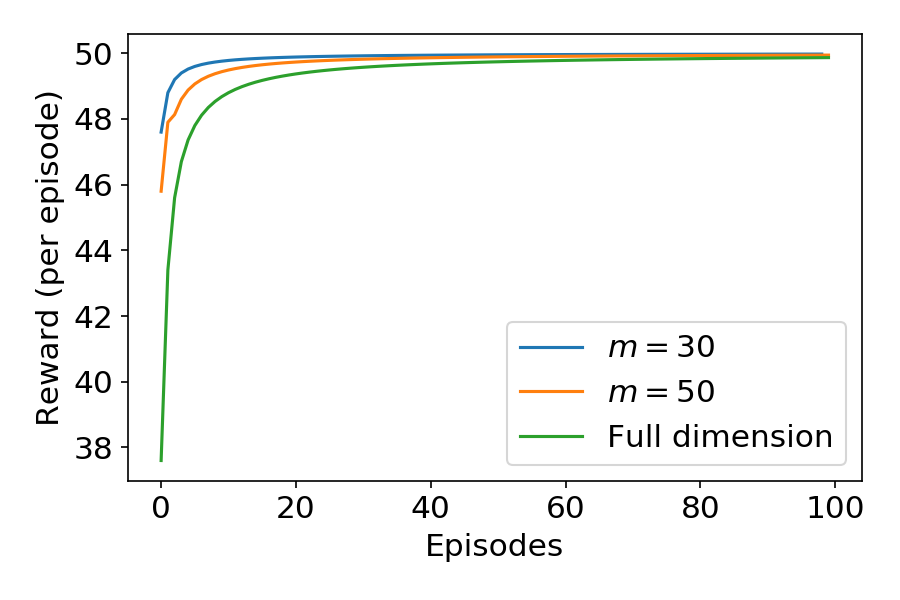}
	\caption{All solid line plots are averaged over 5 runs.
	In both subfigures, $H = 50$. The $m=30$ and $m=50$ plots are Algorithm \ref{alg} with $m = 30$ and $ m=50$ respectively. The ``Full dimension'' curves are OFU algorithms without our PCA projection, which represents previous algorithms \citep{abbasi2011regret}. 
	\textit{Left:} Cost regret (in log-scale) against episode. Cost regret is defined by our LQR formulation: larger cost corresponds to bad cart/pole positions and large velocities, and regret is larger if cost is larger. This shows that empirical results agree with our theory: regret is smaller when $m$ is smaller. 
    \textit{Right:} Reward per episode against episode. The environment gives a unit reward if the cart/pole are in good positions. 
	In terms of both performance metrics, our algorithm achieves better performance than previous methods that do not utilize low-rankness. 
	\label{fig:exp}
	} 
\end{figure*}

\subsection{Part II: Bound the Regret under $\mathcal{E}_{K,\delta} 
$}
\label{sec:part2}








Under event $\mathcal{E}_{K,\delta}$, with the OFU principle, we can decompose the regret as in Proposition \ref{prop:decomp}. 

\begin{proposition}
    \label{prop:decomp} 
    Let $ \widetilde{\Psi}_{k,h} := \Psi_h (\widetilde{M}_k) $ computed by (\ref{eq:def-Ph}). 
    Under event $\mathcal{E}_{K,\delta}$ ($K > K_{\min}^2$), we have 
    \begin{align}
	    &\mathrm{Reg}(K) \le \mathcal{O} \left( H \sqrt{K} \right) \nonumber \\
	    &\qquad + \sum_{k=\left\lceil \sqrt{K} \right\rceil+ 1}^K \sum_{h=1}^{H-1}  \left( \Delta_{k,h} + \Delta_{k,h}^\prime + \Delta_{k,h}^{\prime \prime } \right),  \label{eq:regret-decomp}
    \end{align}    
    where 
    \begin{itemize}[align=left,style=nextline,leftmargin=*,labelsep=\parindent,font=\normalfont]
        \item $\Delta_{k,h} \hspace{-2pt} := \hspace{-2pt} \mathbb{E}_{k,h} \hspace{-2pt} \left[ J_{h+1}^{\pi_k } \left( M_*, \hat{x}_{k,h+1} \right) \right] \hspace{-0pt} - \hspace{-0pt} J_{h+1}^{\pi_k } \left( M_*, \hat{x}_{k,h+1} \right), $ and $\mathbb{E}_{k,h}$ is the expectation conditioning on $\mathcal{F}_{k,h}$ -- all randomness before time $(k,h)$. 
        \item $\Delta_{k,h}' := \left\| \hat{x}_{k,h+1} \right\|_{ \widetilde{\Psi}_{k,h+1} }  - \mathbb{E}_{k,h} \left[  \left\| \hat{x}_{k,h+1} \right\|_{ \widetilde{\Psi}_{k,h+1} }  \right] $,
        \item $\Delta_{k,h}'' := \left\|  M_* \hat{z}_{k,h} \right\|_{ \widetilde{\Psi}_{k,h+1} }   - \left\| \widetilde{M}_k \hat{z}_{k,h} \right\|_{ \widetilde{\Psi}_{k,h+1} } $. 
    \end{itemize}
    
\end{proposition}

Proposition \ref{prop:decomp} is a computational result, and a detailed derivation is in Appendix \ref{app:pf-prop-decomp}. 


Next, in Lemmas \ref{lem:regret-simple-terms} and \ref{lem:bound-delta-pp} below, we bound the the right-hand-side of (\ref{eq:regret-decomp}).

\begin{lemma} \label{lem:regret-simple-terms}
	Under Assumptions \ref{assumption:bounds}-\ref{assumption:state-eigenvalue}, with probability at least $ 1 - 2 \delta $, we have 
	\begin{align*}
		&\left| \sum_{k=1}^K \sum_{h=1}^{H-1} \Delta_{k,h} \right|	 \hspace{-3pt} \le 
		\mathcal{O} \left( \sqrt{ K H^3 \log \frac{2}{\delta} } \right), \quad  \text{and} \\
		& \left| \sum_{k=1}^K \sum_{h=1}^{H-1} \Delta_{k,h}^\prime \right|	\le \mathcal{O} \left( \sqrt{ H K \log \frac{2}{\delta} } \right). 
	\end{align*}
\end{lemma}

We can use the Azuma's inequality to derive Lemma \ref{lem:regret-simple-terms}. The proof of Lemma \ref{lem:regret-simple-terms} is in Appendix \ref{app:pf-regret-simple-terms}. For the regret from $\Delta_{k,h}''$ terms, we use Lemma \ref{lem:bound-delta-pp}.  

\begin{lemma}  \label{lem:bound-delta-pp}
	Let Assumptions \ref{assumption:bounds}-\ref{assumption:state-eigenvalue} hold. Under event $\mathcal{E}_{K, \delta}$ ($K > K_{\min}$),  we have 
    \begin{align}
	    \sum_{k= \left\lceil \sqrt{K} \right\rceil + 1 }^K \sum_{h=1}^{H-1} \Delta_{k,h}'' \le  \widetilde{\mathcal{O}} \left( \left( H^{5/2} + m^{3/2} H \right) \sqrt{K} \right).  \label{eq:bound-delta-pp} 
    \end{align} 
\end{lemma}

The proof for Lemma \ref{lem:bound-delta-pp} is longer. At a high level, this proof uses the following three observations. (i) $ M_* $ are $ {M}_k $ are both rank-$m$. 
(ii) The learned projections across episodes are not too far away. Specifically, $\sum_{k= \left\lceil K \right\rceil + 1}^K \| P_* - P_k \|_2 \le \widetilde{\mathcal{O}} \left( \sqrt{K} \right) $ and $\sum_{ k= \left\lceil K \right\rceil + 1 }^K \| P_K - P_k \|_2 \le \widetilde{\mathcal{O}} \left( \sqrt{K} \right) $. (iii) If everything is projected to the subspace identified by $ P_K $, then this term can be carefully handled by extending previous results for the full rank case (Lemma 7 in \citep{yang2020reinforcement}). More details on proving this Lemma can be found in Appendix \ref{app:pf-bound-delta-pp}.

Now we insert Lemmas \ref{lem:regret-simple-terms} and (\ref{lem:bound-delta-pp}) into (\ref{eq:regret-decomp}) and get, under event $\mathcal{E}_{K,\delta}$, 
$
	Reg (K) =  \widetilde{\mathcal{O}} \left( \left( H^{5/2} + m^{3/2} H \right) \sqrt{K} \right),
$
which proves Theorem \ref{thm}.

\section{Experiments}
\vspace*{-3pt}
\label{sec:exp} 
In the section, we empirically study Algorithm \ref{alg} by deploying it to the \textit{Cart-Pole} problem \citep{brockman2016openai}. Our results (Figure \ref{fig:exp}) show that utilizing low-rankness can significantly improve the regret order. It is worth-noting that controlling \textit{Cart-Pole} from pixels is not easy \citep{lillicrap2015continuous}. In our study of the \textit{Cart-Pole} problem, the state space is velocities (of the cart and the pole tip) together with pixels (images describing the \textit{Cart-Pole} environment). To deploy our algorithm, we first formulate the \textit{Cart-Pole} control as an LQR problem. Specifically, we assume the state transitions follow a linear model. Also, we let the quadratic cost penalize both bad cart/pole positions (bad pixels values) and large velocities. 
For the performance measure, we study both (1) costs from our LQR formulation, and (2) rewards from the \textit{Cart-Pole} environment. For (1), the costs are computed from our LQR formulation. For (2), the environment gives a unit reward if the cart and pole are in good positions, and a zero reward otherwise. The results in (1) from LQR formulation empirically verify our theoretical analysis -- smaller $m$ values give smaller regrets in the LQR formulation. The results in (2) show that our algorithm solves the \textit{Cart-Pole} problem faster than previous methods \citep{abbasi2011regret}. 
We also empirically verify Assumption \ref{assumption:state-eigenvalue}: In common problems, the starting states are on a low-rank space. This study is summarized in Figure \ref{fig:eig}. 
More details on experiment setup are in Appendix \ref{app:exp}. 

\begin{figure}[H]
    \centering
    \includegraphics[width = 0.45\textwidth]{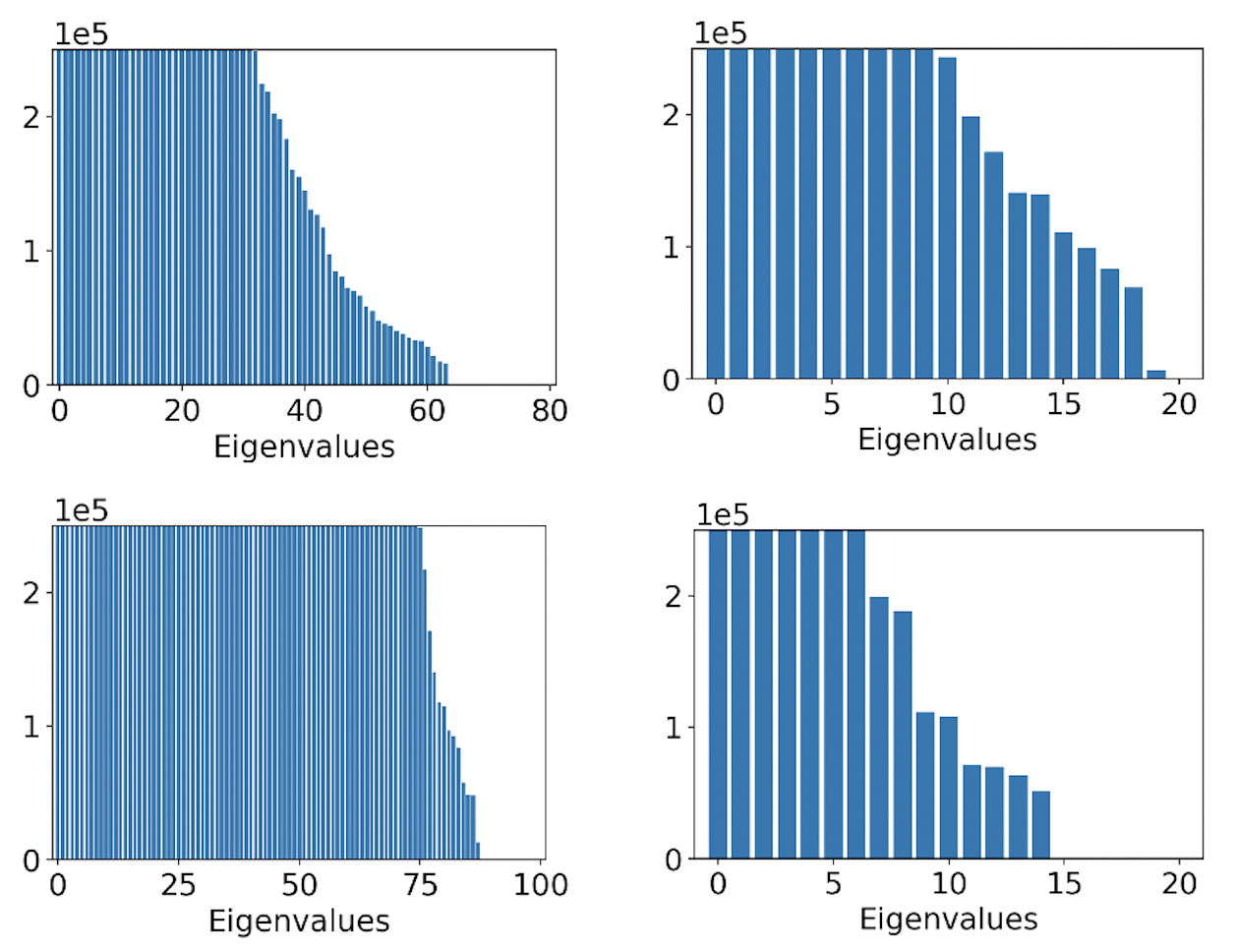}
    \caption{The four bar plots are eigenvalue distributions of starting state covariance (unnormalized) in \texttt{OpenAI Gym} problems \cite{brockman2016openai}. This shows that Assumption \ref{assumption:state-eigenvalue} is empirically true. For the four barplots, upper left: Pendulum, upper right: Acrobot, lower left: LunarLander, lower right: Cart-Pole.\label{fig:eig} } 
    \vspace{-12pt}
\end{figure}
\section{Conclusion}
\vspace*{-3pt}
In this paper we provide a provably efficient reinforcement learning algorithm for controlling LQR systems with unknown dynamics. 
We show that even if the states of the system are of high dimension, our algorithm learns efficiently as long as the system has some intrinsic low-dimensional representation, i.e., the states transition happens in a low-dimensional subspace. 
Our algorithm leverages online LQR control and low-rank approximation techniques to achieve balanced exploration and exploitation inside the low-dimensional subspace.
Numerical studies demonstrate the efficacy of our approach.

\bibliographystyle{apalike} 
\bibliography{biblio} 

\appendix
\onecolumn



\section{Notations and Algorithm Details}


\label{app:notation}
\renewcommand{\arraystretch}{1.2}
\begin{table*}[h!]
\caption{List of notations}
\begin{tabular}{l|l}
\hline
\textbf{Symbol} & \textbf{Definition} \\ \hline
\hline
$H,K$  & horizon (steps per episode), number of episodes\\

$h,k$  & index of horizon, episode\\

$\delta$  & parameter that controls event probabilities\\

$d,d_u,m$  & dimension of state, dimension of control, true state rank\\ 

$Q_h, \; R_h$  & positive definite matrices for state cost and control cost at $h \in [H-1]$ \\ 

$Q_H, \; R_H$  & positive definite matrices for state cost at $H$, $R_H = 0$. \\ 

$M_* = [A_* \quad B_*]$ & true transition dynamics \\  

$\hat{x}_{k,h}, \; u_{k,h}, \; w_{k,h}, $  & (observed) state, control, noise at $(k,h)$, $\hat{z}_{k,h} = [ \hat{x}_{k,h}^\top \quad u_{k,h}^\top ]^\top$\\ 

$ \hat{z}_{k,h}$  & (observed) state-control pair at $(k,h)$\\

${x}_{k,h}, \; z_{k,h} $ & underlying state (noise removed), underlying state-control at $(k,h)$ \\ 

$\Psi_h (M)$ & defined in (\ref{eq:def-Ph}) \\

$\psi_h $ & cost from noise (given transition $M$) in Eq. \ref{eq:J-from-psi} \\

$J_h^*( M, x )$ & cost under optimal policy for system governed $M$, at state $x$ and step $h$.  \\ 

$K_h( M )$ & matrix that defines optimal control law at step $h$ for system $M$ (Eq. \ref{eq:control-law}) \\ 

$\hat{X}_k, \; \hat{X}_k^{\text{next}}$ & $\hat{X}_{k} := 
	\left[ \hat{x}_{h', k'} \right]_{ \substack{ h' \in [ H-1] \\ k' \in [k-1] } }, \; 
	\hat{X}_{k}^{\mathrm{next}} := 
    \left[ \hat{x}_{h', k'} \right]_{ \substack{ h' \in [2, H] \\  k' \in [ k-1] }} $ \\ 

$ U_k$ & $ U_k := \left[ {u}_{h', k'} \right]_{ \substack{ h' \in [2, H] \\  k' \in [ k-1] }}$ \\     

$\hat{Z}_k $ & 
    $  \hat{Z}_{k} := 
	\begin{bmatrix}  \hat{X}_k  \\  U_k  \end{bmatrix}  $ \\ 
	
${X}_k, \; {Z}_k $ &
    $X_k := \left[ {x}_{h', k'} \right]_{ \substack{ h' \in [ H-1] \\ k' \in [k-1] } }, \; {Z}_{k} := 
	\begin{bmatrix}  {X}_k  \\  U_k  \end{bmatrix} $ \\
$P_k$ & projection matrix at $k$ (from top $m$ left singular values of $\hat{X}_{k-1}$)  \\
$P_k^{\mathrm{aug}}$ &  $P_k^{\mathrm{aug}} = \begin{bmatrix}  P_k & 0_{d \times d_u} \\ 0_{ d_u \times d} & I_{d_u}  \end{bmatrix}$ \\

$\widetilde{V}_k , \; V_k $ &  $ \widetilde{V}_k := \hat{Z}_k \hat{Z}_k^\top + I_{d + d_u} ,\quad V_k = P_k^{\mathrm{aug}} \widetilde{V}_k P_k^{\mathrm{aug}} $ \\

$\bar{Z}_k ,\; \bar{X}_k^{\text{next}} $ &  $ \bar{Z}_k := P_k^{\text{aug}} \hat{Z}_k , \quad P_k \hat{X}_k^{\text{next}} $ \\

${M}_k  $ & $M_k^\top = \left( V_k^\dagger \bar{Z}_k  \bar{X}_k^{\mathrm{next}} \right)^\top  $, the estimation of $M_*$ at $k$ \\


$\mathcal{C}_1^{(k)}, \; \mathcal{C}_2^{(k)} $ & confidence set perpendicular to $P_k^{\mathrm{aug}}$, parallel to $P_k^{\mathrm{aug}}$ at $k$ (Eq. \ref{eq:con-proj-precise}, \ref{eq:con-noise-precise}) \\

$ \mathcal{C}^{(k)} $ & confidence at $k$, $\mathcal{C}^{(k)} := \mathcal{C}_* \cap \mathcal{C}_1^{(k)} \cap \mathcal{C}_2^{(k)}$ \\

$ G_{k,\delta}, \beta_{k,\delta} $ & radius of $\mathcal{C}_1^{(k)}$, $\mathcal{C}_2^{(k)}$, precisely defined in Eq. \ref{eq:def-Gk-Cmax}, \ref{eq:def-beta} \\

$\widetilde{M}_k  $ & optimistic estimation of $M_*$ at episode $k$ \\

$\widetilde{\Psi}_{k,h} $ & $\Psi_h(M)$ quantity computed with $\widetilde{M}_k$ \\

$\sigma^2$ & noise covariance $\mathbb{E} [w_{k,h} w_{k,h}^\top] = \sigma^2 I_d$ (assumed for readability) \\

$\lambda_-$ ($\lambda_- > 0$) & bound on $m$-th eigenvalue of start covariance: $ \lambda_m \hspace{-3pt} \left( \mathbb{E} [ x_{k,1} x_{k,1}^\top ] \right) \hspace{-2pt} \ge \hspace{-2pt} \lambda_-$ \\ 

$C$ & bound on $\| Q_h \|_2$ and $\| R_h \|_2$: $ \| Q_h \|_2 \le C $ and $\| R_h \|_2 \le C$ \\

$C_w$ & bound on $\| w_{k,h} \|_2$, $ \| w_{k,h} \|_2 \le C_w$ \\

$C_{\max}$ & constant $C_{\max} := 4 C_w + 2 \sqrt{2} C_w^2 $ \\

$K_{\min}  $ & warm-up period, $ K_{\min} := \Kmin $ \\


\hline
\end{tabular}
\label{table:notation}
\end{table*}

\newpage
\subsection{Precise Definition for $\mathcal{C}^{(k)}$} 

For $k > K_{\min} := \Kmin$, $\mathcal{C}^{(k)} := \mathcal{C}_* \cap \mathcal{C}_1^{(k)} \cap \mathcal{C}_2^{(k)}$, where 
\begin{align}
    &\mathcal{C}_* := \left\{ M = [A \quad B] :  \| M\|_2 \le 1 \right\} \label{eq:def-cstar} \\
    &\mathcal{C}_1^{(k)} := 
    	\left\{ M : 
    		\left\| 
        	 (M - M_k) (I_{d + d_u }  -  {P}_k^{\mathrm{aug}}) 
        	\right\|_2 \lesssim  
        	 G_{k,\delta} 
    	\right\}, \label{eq:con-proj-precise} 
    	\\
    &\mathcal{C}_2^{(k)} := \left\{ M  : \left\| ( M - M_k ) V_k^{1/2}  \right\|_2^2 \lesssim \beta_{k,\delta}  \right\},
    \label{eq:con-noise-precise} \\
    & G_{k,\delta} := \G{k}{H}{\delta}, \quad C_{\max} := 4 C_w + 2 \sqrt{2} C_w^2, \label{eq:def-Gk-Cmax} \\
    & \beta_{ k,\delta} :=  1 + 4 C^2 G_{k ,\delta}^2 Hk + G'_{k,\delta}, \quad 
    G'_{k,\delta} := \Gp{k}.  \label{eq:def-beta} 
\end{align} 

Note that when $ k > K_{\min} $, we have $ G_{k,\delta} > 0 $, thus $ C_1^{(k)} $ is well-defined for $k > K_{\min}$. 

\subsection{A Note on Constants}

\textbf{\textit{Note.}} Throughout the proof, we will overload notations to use $C$ to denote all constants. For example, we will simultaneously say $ \| M_* \| \le C $, $ 2 \| M_* \| \le C$, and  $\| M_*^2 \| \le C $. The only exception is that we use $ C_w $ to denote the bound on noise. Also, this constant $C$ does not depend on $K$, $H$, $\lambda_-$, $\delta$ or $ C_w $, and will be omitted in $\mathcal{O} \( \cdot \)$ notations.

\section{Preparation: Computational Propositions}

\begin{proposition} \label{prop:helper1}
	Recall 
	$	\widetilde{V}_{{k}} = I_{d + d_u} + \hat{Z}_k \hat{Z}_k^\top  \quad \text{and} \quad  V_k = P_k^{\mathrm{aug}} \widetilde{V}_{{k}} P_k^{\mathrm{aug}}. $
	We have 
	\begin{align}
	    V_k^\dagger = P_k^{\mathrm{aug}} \widetilde{V}_{k}^{-1} P_k^{\mathrm{aug}} = \left( \bar{Z}_k \bar{Z}_k^\top  + P_k^{\mathrm{aug}} \right)^{\dagger}, 
	\end{align}
	where $ \bar{Z}_k = P_k^{\mathrm{aug}} \hat{Z}_k $.
\end{proposition} 

\begin{proof}
    
    Let $ L_k^{\mathrm{aug}} $ be the matrix of orthonormal columns such that $ L_k^{\mathrm{aug}} \left( L_k^{\mathrm{aug}} \right)^\top = P_k^{\mathrm{aug}} $. 
    Then one has, 
    \begin{align*}
        V_k^\dagger &= \left( P_k^{\mathrm{aug}} \widetilde{V}_{k} P_k^{\mathrm{aug}} \right)^\dagger  \\
        &= \left( L_k^{\mathrm{aug}} \left( L_k^{\mathrm{aug}} \right)^\top \widetilde{V}_{k} L_k^{\mathrm{aug}} \left( L_k^{\mathrm{aug}} \right)^\top \right)^\dagger \\
        &= \left( \left( L_k^{\mathrm{aug}} \right)^\top \right)^\dagger \left( \left( L_k^{\mathrm{aug}} \right)^\top \widetilde{V}_{k} L_k^{\mathrm{aug}}  \right)^\dagger \left( L_k^{\mathrm{aug}} \right)^\dagger \\
        &= L_k^{\mathrm{aug}} \left(  \left( L_k^{\mathrm{aug}} \right)^\top \widetilde{V}_{k} L_k^{\mathrm{aug}}  \right)^\dagger \left( L_k^{\mathrm{aug}} \right)^\top
        \\
        &= L_k^{\mathrm{aug}} \left(  L_k^{\mathrm{aug}}  \right)^\dagger \widetilde{V}_{k}^\dagger
        \left( \left( L_k^{\mathrm{aug}} \right)^\top \right)^\dagger \left( L_k^{\mathrm{aug}} \right)^\top 
         \\
        &= L_k^{\mathrm{aug}} \left(  L_k^{\mathrm{aug}} \right)^\top \widetilde{V}_{k}^{-1}
        L_k^{\mathrm{aug}} \left( L_k^{\mathrm{aug}} \right)^\top  \\
        &= P_k^{\mathrm{aug}} \widetilde{V}_{k}^{-1} 
        P_k^{\mathrm{aug}} 
    \end{align*} 
    
    Also we have, 
    \begin{align*}
        V_k = P_k^{\mathrm{aug}} \widetilde{V}_{{k}} P_k^{\mathrm{aug}} &= P_k^{\mathrm{aug}} \left( \hat{Z}_k \hat{Z}_k^\top + I_{d+d_u} \right) P_k^{\mathrm{aug}} \\
        &= P_k^{\mathrm{aug}} \hat{Z}_k \hat{Z}_k^\top P_k^{\mathrm{aug}}  + P_k^{\mathrm{aug}} \\
        &= \bar{Z}_k \bar{Z}_k^\top + P_k^{\mathrm{aug}}. 
    \end{align*}
    
    
    
    
    
    
    
    
    
    
\end{proof}

\begin{proposition} \label{prop:helper2}
	Recall 
	\begin{align*}
		\widetilde{V}_{{k}} := \Vtil{k} . 
	\end{align*}
	Fix $K$. Let $ L_K^{aug} $ be the matrix of orthonormal columns such that $ L_K^{\mathrm{aug}} \left( L_K^{\mathrm{aug}} \right)^\top = P_K^{\mathrm{aug}} $. For $k \in [1,K]$, let $V_{K,k} :=  P_K^{\mathrm{aug}} \widetilde{V}_{k} P_K^{\mathrm{aug}} $, and let $D_{K,k} := \left(L_{K}^{\mathrm{aug}} \right)^\top \widetilde{V}_{k} L_{K}^{\mathrm{aug}} $. We have 
	\begin{align} 
		V_{K,k-1}^{\dagger} = L_{K}^{\mathrm{aug}}  D_{K,k-1}^{-1} \left( L_{K}^{\mathrm{aug}} \right)^\top . \label{eq:helper-prop2}
	\end{align} 
\end{proposition} 

\begin{proof}
    Using the similar argument for Proposition \ref{prop:helper1}, 
    we can prove this proposition. 
\end{proof}

\begin{proposition} 
    \label{prop:low-rank-control} 
	Let Assumption \ref{assumption:low-dimension} be true. Let $P_*$ be the true projection matrix  for system $M_* = [A_* \quad B_*]$. Let     $P_*^{aug} := 
     	\begin{bmatrix} 
    		{P}_* & 0_{d \times d_u} \\ 0_{d_u \times d } & I_{d_u } 
    	\end{bmatrix}$. Then we have $M_* = P_* M_* = M_* P_*^{\mathrm{aug}} =  P_* M_* P_*^{\mathrm{aug}}. $ 
\end{proposition} 

\begin{proof}
    By Assumption \ref{assumption:low-dimension}, the true projection matrix for system $M_* = [A_* \quad B_*]$ satisfies  $M_* = [ P_* A_* P_* \quad  P_* B_*]$. Thus, by using $P_* = P_* P_*$, we have 
    \begin{align*}
        M_* = [A_* \quad B_* ] = [ P_* A_* P_* \quad P_* B_* ] = P_* [ P_* A_* P_* \quad P_* B_* ] = P_* M_*, 
    \end{align*}
    \begin{align*}
        M_* P_*^{\text{aug}} = [ P_* A_* P_* \quad P_* B_* ] P_*^{\text{aug}} = [ P_* A_* P_* P_* \quad P_* B_* I_{d_u} ] = M_*. 
    \end{align*}
    
\end{proof}

\section{Well-Definedness of the Algorithm} \label{app:well-defined}


\begin{propositionn}[\ref{prop:bounded-Mk}]
	The regions enclosed by $\mathcal{C}^{(k)}$ ($k \in ( K_{\min}, K]$) are closed and bounded. Also, under event $\mathcal{E}_{K,\delta}$,   $\mathcal{C}^{(k)}$ is non-empty. 
\end{propositionn}
\begin{proof}
    
	\textbf{$\mathcal{C}^{(k)}$ is closed and bounded. }
	It is clear that the regions are closed. Also by definition, $\mathcal{C}^{(k)}$ is bounded since $ \mathcal{C}_* $ is bounded.  

    \textbf{$\mathcal{C}^{(k)}$ is non-empty (with high probability). } Under event $\mathcal{E}_{K,\delta}$, by Lemma \ref{lem:high-prob}, $M_* \in \mathcal{C}^{(k)}$, which shows $\mathcal{C}^{(k)}$ is non-empty.

	
\end{proof}



\begin{lemman}[\ref{lem:psd}]
    The matrix $\Psi_h (M) $ is positive semi-definite for any $h \in [H]$ and $M$, provided that $Q_h , R_h $ are positive definite. 
\end{lemman} 

\begin{proof}
	By Bellman optimality, we know the optimal cost is given by 
	\begin{align*}
		&J_h^* (M, x) = \min_a \left\{ x^\top Q_h x + u^\top R_h u + \overline{J_{h+1}^{*}}  \right\}, \qquad \text{where} \\
		&\overline{J_{h+1}^{*}}:= \mathbb{E}_{w_h}[J_{h+1}^*(M, A x + B  u + w_ h)] .
	\end{align*}
	Solving this dynamic programming problem gives (pp.150, Chapter 4, Vol I; pp. 229, Chapter 5, Vol. I \cite{bertsekas2004dynamic})
	\begin{align*}
    	J_{h}^* (M , x) = x^\top \Psi_h  x + \psi_h ,
	\end{align*} 
	where 
	\begin{align} 
		& x^\top \Psi_H (M) x = 	 x^\top Q_H x, \nonumber \\
	    & x^\top \Psi_h (M)  x = \min_u 
	    \bigg[ 
	    x^\top Q_h x + u^\top R_h u 
	     +
	    \left( A x + B u \right)^\top 
	     \Psi_{h+1} (M)
	    \left( A x + B u \right) 
	    \bigg], \;\; h \hspace{-3pt} < \hspace{-3pt} H. \label{eq:for-PSD} 
	\end{align} 
	
    First, we know from definition that $\Psi_H (M) = Q_H$ is positive definite. Inductively, given $ \Psi_{h+1} (M) $ positive semi-definite, we have from (\ref{eq:for-PSD}) that $x^\top \Psi_{h+1} (M) x \ge 0$, for any $x$. This is because minimization preserves non-negativity. 
    This shows that $ \Psi_h (M) $ is positive semi-definite for all $h$. 
\end{proof} 


\subsection{Boundedness Results}  
\label{app:boundedness} 
    
Based on Assumptions~\ref{assumption:bounds}-\ref{assumption:state-eigenvalue}, the matrices $\widetilde{ \Psi}_{h} (M)$ (defined in Eq.~\ref{eq:def-Ph}) for any $M \in \mathcal{C}^{k}$ ($k \in (K_{\min}, K]$) are bounded. Also, the states, and controls are bounded. 


\begin{proposition} \label{prop:Psi-bounded}
    For all $k \in (K_{\min}, K]$ and $h \in [H]$, there exists a constant $C$, such that, for all $M \in \mathcal{C}^{(k)}$, 
    \begin{align} 
        \left\| \Psi_h (M) \right\|_2 \le C \quad \text{ and } \quad \|  \mathcal{K}_h ( M ) \|_2 \le C. 
    \end{align}
\end{proposition}

\begin{proof}
    Recall for $M = [A \quad B]$, the quantities $ \Psi_h (M)$ are recursively computed by
    \begin{align}
        \Psi_H (M) &:= Q_H, \nonumber \\
        \Psi_{h-1}(M) &:= Q_h \hspace{-2pt} + \hspace{-2pt} A^\top \Psi_{h} (M) A \hspace{-2pt} - \hspace{-2pt} A ^\top \Psi_{h} (M) B \left( R_h \hspace{-3pt} + \hspace{-3pt} B^\top \Psi_{h} (M) B \right)^{-1} \hspace{-3pt} B^\top \hspace{-2pt} \Psi_{h} (M) A, \text{ for } h < H.  
        \label{eq:dummy-helper2}
    \end{align}
    
    
    Any $M \in \mathcal{C}^{(k)}$ is of bounded norm. 
    This is because $ \mathcal{C}_* $ only includes matrices with norm smaller than $C$. 
    
    By positive definiteness of $R_h$, $Q_h$, $M = [A \quad  B]$, we know that $\Psi_{h-1} (M)$ is of bounded norm that is independent of $H$ by 
    eigenvalue bounds on Riccati iterations \citep[Theorem 3.1, item (i), ][]{dai2011eigenvalue}. 
    Note that we can apply this specific result \citep[Theorem 3.1, item (i), ][]{dai2011eigenvalue} even if our system is heterogeneous (in terms of $R_h$ and $Q_h$). 

    Thus, by boundedness of $A$, $B$ and $\Psi_h (M)$, we have 
    \begin{align}
        \|  \mathcal{K}_h ( M ) \|_2 = \| \left( R_h + B ^\top \Psi_{h+1} (M)  B \right) ^{-1} B^\top \Psi_{h+1} (M)  A \|_2 \le C
    \end{align}
    for some constant $C$. 

\end{proof}

\begin{proposition}
    \label{prop:stable}
    Under Assumptions  \ref{assumption:bounds}-\ref{assumption:state-eigenvalue} and event $\mathcal{E}_{K,\delta}$, our algorithm satisfies 
    \begin{align*}
        \left\| \hat{x}_{k,h} \right\|_2 \le 1 \quad \text{and} \quad
        \left\| \hat{z}_{k,h} \right\|_2 \le  C
    \end{align*}
    for all $k \in [\max \( K_{\min}, K_{\min}' \) + 1, K]$ and $h \in [H]$, where $ K_{\min}' : = \left\{ k : G_{k,\delta}^2 + \beta_{k,\delta}^2 \le \frac{ C_w^2 }{C^4} \right\} $. 
\end{proposition}

Note that $ K_{\min}' $ is a constant since $ G_{k,\delta} = \widetilde{\mathcal{O}} \( \frac{1}{\sqrt{k}} \) $ and $\beta_{k,\delta} = \widetilde{\mathcal{O}} \( \frac{1}{\sqrt{k}} \)$. 

\begin{proof}
    
    


    In Eq.~\ref{eq:control-law-learning}, our control is defined by 
    \begin{align} 
        u_{k,h} = \pi_h^{(k)} (x) = \mathcal{K}_h ( \widetilde{M}_k ) \hat{x}_{k,h},\label{eq:dummy-dummy0}
    \end{align}
    
    
    
    Also, since $M_* = [A_* \quad B_*]$, if $ \| x_{k,h} \|_2 \le 1 $,  
    \begin{align} 
        \left\| \hat{x}_{k,h+1} \right\|_2 
        &= 
        \left\| A_* \hat{x}_{k,h} + B_* \mathcal{K}_h (\widetilde{M}_k) \hat{x}_{k,h} \right\|_2 + \left\| B_* \mathcal{K}_h ( \widetilde{M}_* )  \hat{x}_{k,h} - B_* \mathcal{K}_h  (\widetilde{M}_k) \hat{x}_{k,h} \right\|_2 + \left\| {w}_{k,h} \right\|_2 \nonumber \\ 
        &\le r + \left\| B_* \mathcal{K}_h ( \widetilde{M}_* )  \hat{x}_{k,h} - B_* \mathcal{K}_h  (\widetilde{M}_k) \hat{x}_{k,h} \right\|_2 + C_w,  \label{eq:helper-dummy-2} 
    \end{align} 
    where the last line uses Assumption \ref{assumption:bounds} (item (1)). 
    Also by Assumption \ref{assumption:bounds} (item (2)), we have 
    \begin{align*}
        \left\| B_* \mathcal{K}_h ( {M}_* )  \hat{x}_{k,h} - B_* \mathcal{K}_h  (\widetilde{M}_k) \hat{x}_{k,h} \right\|_2 \le \| B_* \|_2 \| \widetilde{M}_k - {M}_* \|_2 \le C^2 \| \widetilde{M}_k - {M}_* \|_2. 
    \end{align*}
    %
    As proved in Section \ref{sec:part1} (and related appendix sections), under event $\mathcal{E}_{K,\delta}$, we use Pythagoras theorem to get  
    \begin{align*}
        \left\| {M}_* - \widetilde{M}_k \right\|_2 \le \sqrt{ G_{k,\delta}^2 + \beta_{k,\delta} }.  
    \end{align*}
    This means, for $k > K_{\min}' : = \left\{ k : G_{k,\delta}^2 + \beta_{k,\delta}^2 \le \frac{ C_w^2 }{C^4} \right\}$, under event $\mathcal{E}_{K,\delta}$, 
    \begin{align*}
        \left\| B_* \mathcal{K}_h ( {M}_* )  \hat{x}_{k,h} - B_* \mathcal{K}_h  (\widetilde{M}_k) \hat{x}_{k,h} \right\|_2 \le C^2 \left\| {M}_* - \widetilde{M}_k \right\|_2 \le C^2 \sqrt{ G_{k,\delta}^2 + \beta_{k,\delta} } \le C_w. 
    \end{align*} 
    By using Assumption \ref{assumption:bounds} (item (3)) in (\ref{eq:helper-dummy-2}) , we get 
    \begin{align*} 
        \left\| \hat{x}_{k,h+1} \right\|_2 \le 1. 
    \end{align*}
    Inductively, this means $ \left\| \hat{x}_{k,h} \right\|_2 \le 1 $ for all $k,h$.  
    Then by Proposition \ref{prop:Psi-bounded}, we know $ \left\| \hat{z}_{k,h} \right\|_2 \le \left\| \mathcal{K}_h (\widetilde{M}_k) \hat{x}_{k,h} \right\|_2  \le C $. 
    
    

    
\end{proof}


\section{Projection Error Analysis (Lemma \ref{lem:projection})} \label{app:projection}
In this section, we prove Lemma \ref{lem:projection}. Our arguments follow the same mechanism as the ones by \cite{vaswani2017finite, lale2019stochastic}. 
We will use the following notation: 
\begin{itemize}
	\item $\mathbb{E}_k :$ the expectation conditioned on all randomness up to time $k - 1$. 
\end{itemize}

	
We first need the following Azuma-Hoeffding inequality for positive semi-definite matrices, which appears as Theorem 3.1 in \cite{tropp2012user}. 

\begin{lemma}[Matrix Chernoff \cite{tropp2011user,tropp2012user}] \label{lem:matrix-hoeffding-small}
	Consider a finite adapted sequence $\{\mathbf{X}_k\}$
of positive-semidefinite matrices with dimension $p$, and suppose that $\lambda_{\max} (\mathbf{X}_k) \le  R$ almost surely. Define the finite series $\mathbf{Y} := \sum_{k} \mathbf{X}_k $ and $\mathcal{Y} := \sum_{k} \mathbb{E}_{k} \mathbf{X}_k $, where $\mathbb{E}_{k}$ is the expectation conditioned on all randomness before $\mathbf{X}_k$. 

Then for all $\mu \ge 0$, 
\begin{align}
	\mathbb{P} \left\{ \lambda_{\min} \hspace{-2pt} \left( \mathbf{Y} \right) \hspace{-2pt} \le \hspace{-2pt} (1 - \delta) \mu \;\;  \text{and} \;\;
		\lambda_{\min } \left( \mathcal{Y} \ge \mu \right) 
		\right\} \le& d \left[ \frac{ e^ { - \delta } }{ (1 - \delta)^{1 - \delta} } \right]^{ \frac{ \mu }{R} } 
\end{align} 
\end{lemma}

We also need the following Lemma by \cite{tropp2012user}. 

\begin{lemma}[Matrix Azuma \cite{tropp2011user, tropp2012user}] \label{lem:matrix-azuma}

Consider a matrix martingale $ \{ \mathbf{Y}_k : k = 0, 1, 2, \cdots \}$ whose
values are self-adjoint matrices with dimension $d$, and let $\{ \mathbf{X}_k : k = 1, 2, 3, \cdots \}$ be the difference
sequence. Assume that the difference sequence satisfies:
\begin{align*}
	\mathbb{E}_{k} \mathbf{X}_k = \mathbf{0},
\end{align*}
and that there exists a deterministic matrix sequence $\{ \mathbf{A}_k: k = 1,2, 3, \cdots \} $ and $\mathbf{X}_k^2 \preceq \mathbf{A}_k^2$ almost surely for $k = 1, 2, 3, .\cdots.$ 

Define the following:
\begin{align}
	\mathbf{Y}_k := \sum_{j=1}^k \mathbf{X}_j,  \qquad \sigma_M^2 := \left\| \sum_{j=1}^k A_k^2 \right\|_2
\end{align} 

Then for all $t \ge 0$, 

\begin{align}
	\mathbb{P} \left\{ \lambda_{\max} \left( Y_k \right) \ge t \right\} \le d \exp \left\{ - \frac{t^2}{ 8 \sigma_M^2 } \right\}. 
\end{align}

\end{lemma}

We also need the following Davis-Kahan sin $\Theta$ theorem. 

\begin{theorem}[Davis-Kahan] \label{thm:projection}
	Let $S, W \in \mathbb{R}^{d \times d}$ be symmetric matrices, and let $\hat{S}_k= S + W$. 
	Let $\lambda_1 \ge  \lambda_2 \ge  \cdots \ge  \lambda_d$ and $\hat{\lambda}_1 \ge  \hat{\lambda}_2 \ge  \cdots \ge  \hat{\lambda}_d$ 
	be the eigenvalues of $S$ and $\hat{S}$ respectively. Define the eigenvalue decompositions of $S$ and $\hat{S}$:
	\begin{align*}
		S &= [L \quad L_0] 
		\begin{bmatrix} \Lambda & 0 \\ 0 & \Lambda_0 \end{bmatrix} [R \quad R_0] \\
		\hat{S}_k&= [ \hat{L} \quad \hat{L}_0] 
		\begin{bmatrix} \hat{\Lambda} & 0 \\ 0 & \hat{\Lambda}_0 \end{bmatrix} [ \hat{R} \quad \hat{R}_0],
	\end{align*}
	where $\Lambda$ (resp. $\hat{\Lambda}$) is the diagonal matrix of the top $m$ eigenvalues of $S$ (resp. $\hat{S}$), and $L$ (resp. $\hat{L}$) is the matrix of the corresponding eigenvectors of $ S $ (resp. $\hat{S}$). 
	
	If $\lambda_m > \hat{\lambda}_{m+1}$, then $\sin \Theta_m$, the sine of the largest principal angle between the column spans of $L$ and $\hat{L}$, can be upper bounded by 
	\begin{align*}
		\sin \Theta_m \le \frac{\left\| \hat{S}_k L - L \Lambda \right\|_2}{  \lambda_m - \hat{\lambda}_{m+1}  }. 
	\end{align*}
	
	In addition, 
	\begin{align*}
		\left\| LL^\top - \hat{L} \hat{L}^\top  \right\|_2 = \sin \Theta_m \le \frac{\left\| \hat{S}_k L - L \Lambda \right\|_2}{  \lambda_m - \hat{\lambda}_{m+1} } . 
	\end{align*}
	
\end{theorem}

Next, we recap Lemma \ref{lem:projection} and provide a proof. 

\begin{lemman}[\ref{lem:projection}]
	Assume the the conditions in Assumption \ref{assumption:state-eigenvalue} hold. Then for $k > K_{\min }$, where
	\begin{align*}
		 \quad K_{\min }:= \Kmin, 
	\end{align*} 
	with probability at least $1 - 3 \delta$, 
	\begin{align*}
		\left\| P_* - P_k \right\|_2	  \le G_{ k, \delta}, 
	\end{align*}
	where 
	\begin{align*} 
		G_{k, \delta} : = \G{k}{H}{\delta}. 
	\end{align*}
\end{lemman}

\begin{proof}
	Let $P_*$ be the true projection matrix of the system. For indexing simplicity, we consider $\left\| P_{k+1} - P_* \right\|_2$. 
	Recall that $\hat{x}_{k^\prime,h} = x_{k^\prime,h} + w_{k^\prime,h}$ where $x_{k^\prime,h}$ lies on a low-rank space. 
	
	Let $L$ be the rank-$m$ orthonormal matrix such that $ P_* = L L^\top $. We set 
	\begin{align*}
		\hat{S}_k :=& \sum_{ \substack{ h < H \\ k^\prime < k }} \hat{x}_{k^\prime,h} \hat{x}_{k^\prime,h}^\top  = \hat{X}_k \hat{X}_k^\top  \\
		S_k :=& \sumHk {x}_{k^\prime,h} {x}_{k^\prime,h}^\top  + LL^\top W_k W_k^\top LL^\top. 
	\end{align*}
	

    Throughout the rest of the proof, for a symmetric matrix $A$, we use $\lambda_i (A)$ to denote the $i$-th largest eigenvalue of $A$. 
    

    \textbf{Step 1: } High probability lower bound on $\lambda_{m } \left( S_k  \right)$. 
    
    \textbf{\textit{High level sketch for this step: }} \textit{Use the Matrix Chernoff Bound (Lemma \ref{lem:matrix-hoeffding-small}) and Assumption \ref{assumption:state-eigenvalue} to show a positive lower bound on $ \lambda_m (S_k)  $.}
    
    Since $x_{k,h}$ lies in the subspace spanned by $L$ (Recall $P_* = LL^\top$), we let $b_{k,h}$ be proper $m$-dimensional vector such that $x_{k,h} = L  b_{k,h}$. Thus 
    \begin{align*} 
        X_k X_k^\top + LL^\top W_k W_k^\top LL^\top = \sumHk L b_{k',h} b_{k',h}^\top L^\top + \sumHk LL^\top w_{k',h} w_{k',h}^\top LL^\top, 
    \end{align*} 
    and
    \begin{align}
        &\lambda_{\min} \left( \sumHk \left[ b_{k,h}  b_{k,h}^\top + L^\top w_{k,h} w_{k,h}^\top L \right] \right) \nonumber \\
        =& \lambda_{m} \left( L \sumHk \left[ b_{k,h}  b_{k,h}^\top + L^\top w_{k,h} w_{k,h}^\top L \right] L^\top \right) \nonumber \\ 
        =& \lambda_{m} \left(  \sumHk \left[ x_{k,h} x_{k,h}^\top + LL^\top w_{k,h} w_{k,h} LL^\top \right] \right) \nonumber \\
        =& \lambda_{m} \left( S_k \right) 
        \label{eq:state-X-to-B} . 
    \end{align}
    
    By Assumption \ref{assumption:state-eigenvalue}, 
    \begin{align}
        \lambda_{\min} ( \mathbb{E}_k [ b_{k,1} b_{k,1}^\top] ) &= \lambda_{m} ( \mathbb{E}_k [ x_{k,1} x_{k,1}^\top] \ge \lambda_-. \label{eq:state-x-to-b}
    \end{align}
    
    By taking conditional expectations and eigenvalues, we have, 
    \begin{align}
        &\lambda_m \left( \sum_{k'=1}^{k-1} \sum_{h=1}^{H-1} \mathbb{E}_{k'} [ L b_{k',h} b_{k',h}^\top L^\top + LL^\top w_{k,h} w_{k,h}^\top LL^\top ] \right) 
        \nonumber \\ 
        =& \lambda_{\min} \left( \sum_{k'=1}^{k-1} \sum_{h=1}^{H-1} \mathbb{E}_{k'} [ b_{k',h} b_{k',h}^\top + L^\top w_{k,h} w_{k,h}^\top L ] \right)  \tag{since $L$ is orthonormal and $P_* = L L^\top$} \\
        \ge& \lambda_{\min} \left( \sum_{k'=1}^{k-1} \sum_{h=1}^{H-1} \mathbb{E}_{k'} [ b_{k',h} b_{k',h}^\top ] \right) + \lambda_{\min} \left( \sum_{k'=1}^{k-1} \sum_{h=1}^{H-1} \mathbb{E}_{k'} [ L^\top w_{k,h} w_{k,h}^\top L ] \right) \tag{by Lidskii inequality} \\
        \ge& k \lambda_- + \lambda_{\min} \left( \sum_{k'=1}^{k-1} \sum_{h=1}^{H-1} \mathbb{E}_{k'} [ L^\top w_{k,h} w_{k,h}^\top L ] \right) \tag{by Eq. \ref{eq:state-x-to-b}} \\
        =& (k-1) \lambda_- + (k-1) (H-1) \sigma^2. \label{eq:dummy}
    \end{align}
	
	Let $\mathcal{F}_k$ be all randomness by end of episode $k-1$. 
	Then $ \left\{ \sum_{h=1}^{H-1} b_{k,h} b_{k,h}^\top  \right\}_k $ are adapted to the filtration $\{ \mathcal{F}_k \}_k$, since $ \left\{ \sum_{h=1}^{H-1} x_{k,h} x_{k,h}^\top  \right\}_k $ are adapted to $\{ \mathcal{F}_k \}_k$ and $b_{k,h}$ are determined by $x_{k,h}$ and the constant matrix $L$. 
	
	By Assumption \ref{assumption:bounds}, $\lambda_{\max} \left( \sum_{h=1}^{H-1} b_{k',h} b_{k',h}^\top \right) = \lambda_{\max} \left( \sum_{h=1}^{H-1} x_{k',h} x_{k',h}^\top \right) \le H $. Then we apply Lemma \ref{lem:matrix-hoeffding-small} to the matrices $ \left\{ \sum_{h=1}^{H-1} b_{k',h} b_{k',h}^\top  \right\}_{k'} $, and set $\delta = k^{-1/4}$ to get 
	\begin{align} 
		&\mathbb{P} \Bigg[ \lambda_{\min } \left( \sum_{k'=1}^{k-1} \sum_{h=1}^{H-1} \left( b_{k',h} b_{k',h}^\top + L w_{k,h} w_{k,h}^\top L^\top \right) \right) \nonumber \\
		&\qquad \qquad \le \left( 1 -  k^{-1/4} \right)  \lambda_{\min} \left( \sum_{k'=1}^{k-1} \mathbb{E}_{k'} \sum_{h=1}^{H-1} \left( b_{k',h} b_{k',h}^\top + L w_{k,h} w_{k,h}^\top L^\top \right) \right)
		\Bigg] 
		\nonumber \\
		\le &
		m \left[ \frac{ \exp \left\{- k^{-1/4}  \right\} }{ \left( 1 - k^{-1/4} \right)^
		{ \left( 1- k^{-1/4}  \right) } } \right]^
		{ \frac{ \lambda_{\min} \left( \sum_{k'=1}^{k-1}  \mathbb{E}_{k'} \left[ \sum_{h=1}^{H-1} \left( b_{k',h} b_{k',h}^\top + L w_{k,h} w_{k,h}^\top L^\top \right) \right] \right) }{ H } }.  
	\end{align} 
	Since, from (\ref{eq:dummy}), 
	\begin{align*}
		\lambda_m \left( \sum_{k'=1}^{k-1} \sum_{h=1}^{H-1} \mathbb{E}_{k'} [ L b_{k',h} b_{k',h}^\top L^\top + LL^\top w_{k,h} w_{k,h}^\top LL^\top ] \right)  \ge (k-1) \lambda_- + (k-1)(H-1) \sigma^2,
	\end{align*}
	we have 
	\begin{align}
		&\mathbb{P} \Bigg[ \lambda_{\min } \left( \sum_{k'=1}^{k-1} \sum_{h=1}^{H-1} \left( b_{k',h} b_{k',h}^\top + L w_{k,h} w_{k,h}^\top L^\top \right) \right) \nonumber \\
		&\qquad \qquad \le \left( 1 -  k^{-1/4} \right)  \left[ (k-1) \lambda_- + (k-1)(H-1) \sigma^2 \right] 
		\Bigg] \nonumber \\ 
		\le 
		&\mathbb{P} \Bigg[ \lambda_{\min } \left( \sum_{k'=1}^{k-1} \sum_{h=1}^{H-1} \left( b_{k',h} b_{k',h}^\top + L w_{k,h} w_{k,h}^\top L^\top \right) \right) \nonumber \\
		&\qquad \qquad \le \left( 1 -  k^{-1/4} \right)  \lambda_{\min} \left( \sum_{k'=1}^{k-1} \mathbb{E}_{k'} \sum_{h=1}^{H-1} \left( b_{k',h} b_{k',h}^\top + L w_{k,h} w_{k,h}^\top L^\top \right) \right) 
		\Bigg] \nonumber \\
		\le & m \left[ \frac{ \exp \left\{- k^{-1/4}  \right\} }{ \left( 1 - k^{-1/4} \right)^
		{ \left( 1- k^{-1/4}  \right) } } \right]^
		{ \frac{ \lambda_{\min} \left( \sum_{k'=1}^{k-1} \mathbb{E}_{k'} \sum_{h=1}^{H-1} \left( b_{k',h} b_{k',h}^\top + U w_{k,h} w_{k,h}^\top U^\top \right) \right) }{ H } } \nonumber \\
		\le & m \left[ \frac{ \exp \left\{- k^{-1/4}  \right\} }{ \left( 1 - k^{-1/4} \right)^
		{ \left( 1- k^{-1/4}  \right) } }  \right]^{ \frac{ (k-1) \lambda_- + (k-1)(H-1) \sigma^2 }{ H } }. \label{eq:dummy2}
	\end{align} 
	We combine (\ref{eq:state-X-to-B}) and (\ref{eq:dummy2}) to get 
	\begin{align*}
		&\mathbb{P} \left[ \lambda_{m } ( S_k  ) \le \left( 1 - k^{-1/4} \right) (k-1) \lambda_- \right] \\
		=&\mathbb{P} \Bigg[ \lambda_{\min } \left( \sum_{k'=1}^{k-1} \sum_{h=1}^{H-1} \left( b_{k',h} b_{k',h}^\top + L w_{k,h} w_{k,h}^\top L^\top \right) \right) \nonumber \\
		&\qquad \qquad \le \left( 1 -  k^{-1/4} \right)  \left[ (k-1) \lambda_- + (k-1)(H-1) \sigma^2 \right] \Bigg] \\
		\le& m \left[ \frac{ \exp \left\{- k^{-1/4}  \right\} }{ \left( 1 - k^{-1/4} \right)^
		{ \left( 1- k^{-1/4}  \right) } }  \right]^{ \frac{ (k-1) \lambda_- + (k-1)(H-1) \sigma^2 }{ H } }.  
	\end{align*}
	Since  $ 
	    \left[ \frac{ \exp \left\{- k^{-1/4}  \right\} }{ \left( 1 - k^{-1/4} \right)^{ \left( 1- k^{-1/4}  \right) } } \right]^{ \frac{ (k-1) \lambda_- + (k-1)(H-1) \sigma^2 }{ H } } \le \exp \left\{ - \frac{ {k}^{1/4} \left( \lambda_- + (H-1) \sigma^2 \right) }{ H } \right\}$, which can be verified by calculus, 
    after some computation and rearrangement, we get 
	\begin{align*}
		&\mathbb{P} \left[ \lambda_{m } ( S_k ) \le \left( 1 -  k^{-1/4} \right)  \left[ (k-1) \lambda_- + (k-1)(H-1) \sigma^2 \right]  \right] \\
		\le& m \exp \left\{ - \frac{ {k}^{1/4} \left( \lambda_- + (H-1) \sigma^2 \right) }{ H  } \right\} . 
	\end{align*}
	Thus for any $\delta \in (0,1)$, with probability at least $1 - \delta$, 
	\begin{align}
		\lambda_{m } \left( S_k  \right) > (k-1) \lambda_- + (k-1) (H-1) \sigma^2 - k^{3/4} H \log \frac{m}{\delta}. \label{eq:step1-res}
	\end{align} 

    \textbf{Step 2: } High probability upper bound on $ \left\| \hat{S}_k- S_k - \mathbb{E}_k^\Sigma \left[ \hat{S}_k- S_k  \right] \right\|_2$ \Bigg($\mathbb{E}_k^\Sigma \left[ \hat{S}_k- S_k  \right]$ defined below\Bigg)
    
    \textit{Note*: the $\mathbb{E}_k^\Sigma \left[ \hat{S}_k- S_k  \right]$ term is defined below in (\ref{eq:def-exp-sum})}. 
    
    For simplicity, we define 
    \begin{align}
        \mathbb{E}_k^\Sigma \left[ \hat{S}_k- S_k \right] := \sum_{k' = 1}^{k-1} \mathbb{E}_{k'} \left[  \sum_{ h = 1 }^{H-1} \hat{x}_{k^\prime,h} \hat{x}_{k^\prime,h}^\top - {x}_{k^\prime,h} {x}_{k^\prime,h}^\top - LL^\top w_{k',h} w_{k1,h}^\top LL^\top  \right],  \label{eq:def-exp-sum}
    \end{align}
    where $\mathbb{E}_{k'}$ is the expectation conditioning on all randomness before episode $k' $.

    \textbf{\textit{High level sketch for this step: }} \textit{Apply the Matrix Azuma Inequality (Lemma \ref{lem:matrix-azuma}) to bound the term $ \left\| \hat{S}_k- S_k - \mathbb{E}_k^\Sigma \left[ \hat{S}_k- S_k  \right] \right\|_2$.}
     
    
    By definition, we have
	\begin{align*}
		\hat{S}_k- S_k
			=& W_k X_k^\top + X_k W_k^\top  + W_k W_k^\top - LL^\top W_k W_k^\top LL^\top
	\end{align*}
	\begin{align}
	    \mathbb{E}_k^\Sigma \left[ \hat{S}_k- S_k\right] 
		=&
			\mathbb{E}_k^\Sigma \left[ W_k X_k^\top + X_k W_k^\top  + W_k W_k^\top - LL^\top W_k W_k^\top LL^\top \right] \nonumber \\
		=&  (H-1) (k-1) \sigma^2 ( I_d - LL^\top ) , \label{eq:exp-sum-to-id}
	\end{align}
	where the notation $\mathbb{E} [w_{k,h}] = 0 $ is defined in (\ref{eq:def-exp-sum}), $\mathbb{E}_k [w_{k^\prime,h} w_{k^\prime,h}^\top] = \sigma^2 I_d $ (Assumption \ref{assumption:noise}), and the last equation uses that $\mathbb{E}_k w_{k^\prime,h} = 0$ (Assumption \ref{assumption:noise}). 
	
	Thus we have, 
	\begin{align*}
		\left\| \mathbb{E}_k^\Sigma \left[ \hat{S}_k- S_k\right] \right\|_2^2
		=&\left\| (H-1) (k-1) \sigma^2 ( I_d - LL^\top )  \right\|_2 \\
		=& (H-1) (k-1) \sigma^2, \\
		\hat{S}_k- S_k- \mathbb{E}_k^\Sigma \left[ \hat{S}_k- S_k \right] 
		=& W_k X_k^\top + X_k W_k^\top + W_k W_k^\top - LL^\top W_k W_k^\top LL^\top \\
		&\quad -  ( H-1 ) (k-1) \sigma^2 (I_d - LL^\top).
	\end{align*}
	
    We write $ P_*^{\bot} := I_d - P_* $. Since $P_* = LL^\top$, from the above expression for $\hat{S}_k- S_k- \mathbb{E}_k^\Sigma \left[ \hat{S}_k- S_k \right]$, we have 
	\begin{align}
		\left\| \hat{S}_k- S_k- \mathbb{E}_k^\Sigma \left[ \hat{S}_k- S_k \right] \right\|_2 =& \left\| W_k X_k^\top + X_k W_k^\top  +  P_*^\bot W_k W_k^\top P_*^\bot  -  (H  - 1) (k-1) \sigma^2 P_*^\bot  \right\|_2 \nonumber \\
		\le&  \lambda_{\max} \left( P_*^\bot W_{k} W_{k}^\top P_*^\bot - \sigma^2 (H-1) (k-1) P_*^\bot \right) \nonumber \\
		&+ \lambda_{\max} \left( W_k X_k^\top + X_k W_k^\top \right), \label{eq:tri-S}
	\end{align} 
	where on the last step we use the triangle inequality. 
	
	Since, for any $(k,h)$, 
	\begin{align}
	    \mathbb{E}_k \left[ x_{k,h} w_{k,h}^\top  \right] = 0, \quad \text{ and } \quad \mathbb{E}_k \left[ P_*^\bot w_{k,h} w_{k,h}^\top P_*^\bot - \sigma^2 P_*^\bot  \right] = 0, 
	\end{align}
	the sequences $ \left\{ P_*^\bot W_{k} W_{k}^\top P_*^\bot - \sigma^2 (H-1) (k-1) P_*^\bot \right\}_k $ and $ \left\{ W_k X_k^\top + X_k W_k^\top \right\}_k $ are both martingale sequences of symmetric matrices. 
	
	From there we use the Matrix Azuma inequality to get: for any $\delta \in (0,1)$, 
    \begin{align}
        &\mathbb{P} \Bigg\{ \lambda_{\max} \left( P_*^\bot W_{k} W_{k}^\top P_*^\bot - \sigma^2 (H-1) (k-1) P_*^\bot \right)  \ge   \sqrt{ 8 C_w^4 H k \log  \frac{d}{\delta}  } \Bigg\} \le \delta,  \\
        &\mathbb{P} \Bigg\{ \lambda_{\max} \left( W_k X_k^\top + X_k W_k^\top \right) \ge  4 C_w \sqrt{ H k \log  \frac{d}{\delta} } \Bigg\} \le \delta . 
    \end{align}  
    
    Then by a union bound and (\ref{eq:tri-S}), we get, for any $\delta \in (0,1)$, with probability at least $1 - 2\delta$, 
    \begin{align}
        \left\| \hat{S}_k- S_k- \mathbb{E}_k^\Sigma \left[ \hat{S}_k- S_k \right] \right\|_2 \le C_{\max} \sqrt{Hk \log \frac{d}{\delta} }, \label{eq:step2-res}
    \end{align}
    where $C_{\max} = 4 C_w + 2 \sqrt{2} C_w^2$. 
    
    \textbf{Step 3: } High probability lower bound on $\lambda_m (S_k ) - \lambda_{m+1} (\hat{S}_k)$ when $k$ is larger than a constant. 
    
    \textbf{\textit{High level sketch for this step: }} \textit{link $ \lambda_{m+1} (\hat{S}_k) $ to $\left\| \hat{S}_k- S_k - \mathbb{E}_k^\Sigma \left[ \hat{S}_k- S_k  \right] \right\|_2 $ and apply results from Step 1 and Step 2. }
    
    Since $\lambda_{m+1}(S_k) = 0$, by Weyl's inequality, we have
    \begin{align}
        \lambda_{m+1} (\hat{S}_k) \le \lambda_{m+1} ({S}_k) + \lambda_1 ( \hat{S}_k- S_k) = \left\| \hat{S}_k- S_k\right\|_2. \label{eq:lam-m+1-to-norm}
    \end{align} 
    
    By triangle inequality, 
    \begin{align}
        \left\| \hat{S}_k- S_k\right\|_2 \le& \left\| \hat{S}_k- S_k - \mathbb{E}_k^\Sigma \left[ \hat{S}_k- S_k \right] \right\|_2  + \left\|  \mathbb{E}_k^\Sigma \left[ \hat{S}_k- S_k \right] \right\|_2 \nonumber \\
        \le & \left\| \hat{S}_k- S_k - \mathbb{E}_k^\Sigma \left[ \hat{S}_k- S_k \right] \right\|_2 + (k-1) (H-1) \sigma^2, \label{eq:triangle-res} 
    \end{align}
    where the last step uses $\mathbb{E}_k^\Sigma \left[ \hat{S}_k- S_k \right] = \sigma^2 (k-1) (H-1) ( I_d - P_*)$. (Recall $\mathbb{E}_k^\Sigma \left[ \hat{S}_k- S_k \right]$ is defined in Eq. \ref{eq:def-exp-sum}.)
    
    From step 1, we have, with probability at least $1 - \delta$, 
    \begin{align} 
        \lambda_m ( S_k ) \ge (k-1) \lambda_- + (k-1) (H-1) \sigma^2 - k^{3/4} H \log \frac{m}{\delta}. 
    \end{align} 
    
    With probability at least $1 - 2\delta$, the results in both Step 1 and Step 2 holds, which gives, 
    \begin{align}
        &\lambda_m ( S_k ) - \lambda_{m+1} ( \hat{S}_k ) \nonumber \\
        \ge& (k-1) \lambda_- + (k-1) (H-1) \sigma^2 - k^{3/4} H  \log \frac{m}{\delta} - \lambda_{m+1} ( \hat{S}_k )  \label{eq:use-step1} \\
        \ge& (k-1) \lambda_- + (k-1) (H-1) \sigma^2 - k^{3/4} H  \log \frac{m}{\delta} - \left\| \hat{S}_k- S_k\right\|_2 \label{eq:use-lam-to-norm} \\
        \ge& (k-1) \lambda_- + (k-1) (H-1) \sigma^2 - k^{3/4} H  \log \frac{m}{\delta} \nonumber \\
        &\qquad - \left\| \hat{S}_k- S_k - \mathbb{E}_k^\Sigma \left[ \hat{S}_k- S_k \right] \right\|_2 - (k-1) (H-1) \sigma^2 \label{eq:use-triangle-res} \\
        =& (k-1) \lambda_- - k^{3/4} H  \log \frac{m}{\delta} - \left\| \hat{S}_k- S_k - \mathbb{E}_k^\Sigma \left[ \hat{S}_k- S_k \right] \right\|_2  \nonumber \\
        \ge& (k-1) \lambda_- - k^{3/4} H  \log \frac{m}{\delta} - C_{\max} \sqrt{Hk \log \frac{d}{\delta} }, \label{eq:use-step2}
    \end{align} 
    where (\ref{eq:use-step1}) uses Step 1, (\ref{eq:use-lam-to-norm}) uses (\ref{eq:lam-m+1-to-norm}), (\ref{eq:use-triangle-res}) uses (\ref{eq:triangle-res}), and (\ref{eq:use-step2}) uses Step 2. 
    
    When $ k > K_{\min } := \Kmin $, the expression in (\ref{eq:use-step2}) is positive. 
    Thus, with probability at least $1 - 2 \delta$, for all  $ k > K_{\min } := \Kmin $,
    \begin{align}
        \lambda_m (S_k) - \lambda_{m+1} (\hat{S}_k) \ge (k-1) \lambda_- - k^{3/4} H \log \frac{m}{\delta} - C_{\max} \sqrt{Hk \log \frac{d}{\delta} } > 0. \label{eq:step3-res}
    \end{align}
    
    \textbf{Step 4: } Final step. 
    
    \textit{\textbf{High level sketch of this step:} Apply Theorem 2 and combine previous steps.}
    
    We use $P_*$ to denote the true projection matrix and it is clear that $P_* = LL^\top $. 
	Then, with probability at least $1 - 3 \delta$, for any $ k > K_{\min} $ 
	\begin{align}
		\left\| P_{k+1} - P_*\right\|_2 
		\overset{\text{\textcircled{1}}}{\le}& \frac{ \left\| \hat{S}_k L - L \Lambda \right\|_2 }{  \lambda_m (S_k) - {\lambda}_{m+1} (\hat{S}_k)  } 
		  \overset{\text{\textcircled{2}}}{=} 
			\frac{ \left\| (\hat{S}_k - S_k ) L \right\|_2 }{  \lambda_m (S_k) -  {\lambda}_{m+1} (\hat{S}_k)  }  \nonumber , 
	\end{align} 
	where \textcircled{1} uses Theorem \ref{thm:projection} (we can use Theorem \ref{thm:projection} since by Step 3, $ \lambda_m (S) > \lambda_{m+1 } (\hat{S}) $ with high probability for all $k > K_{\min}$),  {\textcircled{2}} uses $S = L \Lambda L^\top $.
	
	Applying the triangle inequalities to the above and get: 
	\begin{align}
		\left\| P_{k+1} - P_*\right\|_2  \le& \frac{ \left\| (\hat{S}_k- S_k) L \right\|_2 }{ \lambda_m (S) - \lambda_{m+1} (\hat{S}_k ) } \nonumber \\
		=& 
			\frac{ \left\| \mathbb{E}_k^\Sigma (\hat{S}_k- S_k) L +  (\hat{S}_k- S_k) L - \mathbb{E}_k^\Sigma (\hat{S}_k- S_k) L \right\|_2 }{ \lambda_m (S) - \lambda_{m+1} (\hat{S}_k )  } \nonumber \\
		\le& 
			\frac{ \left\| \mathbb{E}_k^\Sigma (\hat{S}_k- S_k) L \right\|_2 + \left\| (\hat{S}_k- S_k) L - \mathbb{E}_k^\Sigma (\hat{S}_k- S_k) L \right\|_2 }{ \lambda_m (S) - \lambda_{m+1} (\hat{S}_k )  } \nonumber \\
		\le& 
			\frac{ \left\| \mathbb{E}_k^\Sigma (\hat{S}_k- S_k) L \right\|_2 + \left\| (\hat{S}_k- S_k) - \mathbb{E}_k^\Sigma (\hat{S}_k- S_k) \right\|_2 \left\| L \right\|_2 }{ \lambda_m (S) - \lambda_{m+1} (\hat{S}_k )  } \nonumber \\
		\le& 
			\frac{ \left\| \mathbb{E}_k^\Sigma (\hat{S}_k- S_k) L \right\|_2 + \left\| (\hat{S}_k- S_k) - \mathbb{E}_k^\Sigma (\hat{S}_k- S_k) \right\|_2 }{ \lambda_m (S_k) - \lambda_{m+1} (\hat{S}_k )  }, \label{eq:check-pt}
	\end{align}
	where in the last step we use $ \| L \|_2 = 1$. 
	
	We use (\ref{eq:exp-sum-to-id}) to get, 
	\begin{align}
	    \mathbb{E}_k^\Sigma \left[ \hat{S}_k- S_k\right] L = (H-1) (k-1) \sigma^2 ( I_d - LL^\top ) L = 0. \label{eq:helper}
	\end{align}
	
	We then plug (\ref{eq:helper}) into (\ref{eq:check-pt}) to get, with probability at least $1 - 3 \delta$, 
	\begin{align*}
	    \left\| P_{k+1} - P_*\right\|_2  \le & \frac{ \left\| \mathbb{E}_k^\Sigma (\hat{S}_k- S_k) L \right\|_2 + \left\| (\hat{S}_k - S_k) - \mathbb{E}_k^\Sigma (\hat{S}_k- S_k) \right\|_2 }{ \lambda_m (S) - \lambda_{m+1} (\hat{S}_k ) }  \\
	    \le& \frac{ \left\| (\hat{S}_k- S_k) - \mathbb{E}_k^\Sigma (\hat{S}_k- S_k) \right\|_2 }{ \lambda_m (S) - \lambda_{m+1} (\hat{S}_k ) } . 
	\end{align*} 
	
	Then by Step 2 (Eq. \ref{eq:step2-res}) and Step 3 (Eq. \ref{eq:step3-res}), we have with probability at least $1 - 3 \delta$, for any $ k > K_{\min} := \Kmin$, 
	\begin{align}
	    \left\| P_{k+1} - P_*\right\|_2  \le \frac{C_{\max} \sqrt{Hk \log \frac{d}{\delta} } }{ (k-1) \lambda_- - k^{3/4} H \log \frac{m}{\delta} - C_{\max} \sqrt{Hk \log \frac{d}{\delta} } }. 
	\end{align}
	
	This concludes the proof. 
\end{proof}

\section{Rank-deficit Self-normalized Processes (for Lemma \ref{lem:high-prob})}
\label{app:rank-deficit} 

In this section, we prove Lemma \ref{lem:low-rank-concentration}, which is used to prove Lemma \ref{lem:high-prob}. We first need the following result. 

\begin{lemma} [Lemma 8 in \cite{abbasi2011improved}] 
\label{lem:prepare-self-normalized}
    Let $Y_1, Y_2, \cdots \in \mathbb{R}^d$ be vector-valued random variables, and let $\{ \eta_s \}_s$ be real-valued random variables. Let 
    $
    \mathcal{F}_t := \sigma \left( Y_1, Y_2, \cdots, Y_{t-1}, \eta_1, \eta_2, \cdots, \eta_{t-1} \right),
    $
    and let $\{ \eta_s \}_s$ be conditionally $R$-sub-Gaussian: 
    \begin{align*} 
        \forall \lambda \in \mathbb{R}, \qquad \mathbb{E} [ \lambda \eta_t | \mathcal{F}_t ] \le \exp \left\{ \frac{R^2 \lambda^2}{2} \right\}. 
    \end{align*} 
    
    Let $\lambda \in \mathbb{R}^d$ be arbitrary and consider for any $t \ge 0$, 
    \begin{align}
        M_t^\lambda := 
        \exp \left( \sum_{s=1}^t \left[ \frac{ \eta_s \left< \lambda, Y_s \right> }{R} - \frac{1}{2} \left< \lambda, Y_s \right>^2 \right] \right). \label{eq:def-m}
    \end{align}
    Then for $\tau$ a stopping time with respect to the filtration $\{ \mathcal{F}_t\}_{t=0}^\infty$, $M_\tau^\lambda$ is almost surely well-defined and 
    \begin{align*}
        \mathbb{E} [ M_\tau^\lambda ] \le 1. 
    \end{align*}
\end{lemma}


\begin{lemma}
    \label{lem:low-rank-concentration}
    Consider a process $\{ y_s \}_s$ in $\mathbb{R}^d$. Let $\{ \eta_s \}_s$ be a  mean-zero process on the real line. Define $\mathcal{F}_t = \sigma ( y_1, \cdots, y_{t-1}, \eta_1, \cdots, \eta_{t-1} ) $ and let $\eta$ be conditionally $R$-sub-Gaussian (with respect to $\mathcal{F}_t$). Let $Y_t := \left[y_s \right]_{s \in [1,t]}$ be the matrix whose columns are $y_s$. 
    Let $t$ be a stopping time and 
    let $\{P_s\}_s$ be a sequence of rank-$m$ projection matrices such that $P_t$ is $\mathcal{F}_t$-measurable. 
    Let $\bar{Y}_t := P_t Y_t$, let $\bar{V}_t := \bar{Y}_t \bar{Y}_t^\top $, $ V_t := \bar{Y}_t \bar{Y}_t^\top + P_t $
    and let $ S_t := \bar{Y}_t \xi_t$, where $\xi_t = \left[ \eta_s \right]_{s \in [1,t]}^\top$ is the vector formed by $\eta_s$. 
    Then, for any $\delta > 0$, 
    \begin{align}
        \mathbb{P} \left(   
	         \left\| S_t \right\|_{ {V}_t^\dagger }^2
	         > 2 R^2  \log \left( \delta^{-1} \sqrt{ \frac{ \det^* V_t }{  \det^* P_t } }\right)
        	\right) \le \delta,
    \end{align}
    where $\left\| \cdot \right\|_{{V}_t^\dagger}$ is the semi-norm induced by ${V}_t^\dagger$, and $\det^*$ is the pseudo-determinant operator. 
\end{lemma}

\begin{proof}

To prove this lemma, we first need Lemma \ref{lem:prepare-self-normalized} by \cite{abbasi2011improved}, and use the techniques by \cite{de2009multivariate, abbasi2011improved}. 

    For a rank-deficit symmetric positive semi-definite matrix $ \Sigma^\dagger $ and vector $\mu \in span (\Sigma)$, consider the singular (or degenerate) multi-variate Gaussian distribution $ \mathcal{N} ( \mu, \Sigma^\dagger)$ whose mean is ${\mu}$ and covariance is $\Sigma^\dagger$. 
    The density of this distribution is defined only over $span(\Sigma)$. 
    For $x \in span(\Sigma)$, the density function of this degenerate multi-variate Gaussian is 
    \begin{align} 
        f_{\mu, \Sigma} (x) &= \frac{1}{ \sqrt{ {\det}^* \left(2 \pi \Sigma^\dagger \right)  } } \exp \left( - \frac{1}{2} (x - \mu )^\top \Sigma (x - \mu) \right) \nonumber \\
        &= \frac{1}{ \sqrt{ {\det}^* \left(2 \pi \Sigma^\dagger \right)  } } \exp \left( - \frac{1}{2} \left\| x - \mu \right\|_{ \Sigma } \right), \; 
        \label{eq:degerate-density} 
    \end{align} 
    where $\det^*$ is the  pseudo-determinant. (For PSD matrices, pseudo-determinant gives the product of positive eigenvalues.) 
    For $\lambda \in span(\bar{V}_t)$, let
    \begin{align*} 
        M_t^\lambda = \exp \left( \sum_{s=1}^t \left[ \frac{ \eta_s \left< \lambda,  P_t y_s \right> }{R} - \frac{1}{2} \left< \lambda,  P_t y_s \right>^2 \right] \right), 
    \end{align*} 
    as defined above in (\ref{eq:def-m}). 
    
    Let $P_t$ be the projection matrix at $ t $. Let $f_{0,P_t}$ be the density for $\mathcal{N} \left( 0, P_t \right)$. 
    With respect to this density $f_{0,P_t}$, we have
    \begin{align}
        &\int_{span({P}_t)} M_t^\lambda f_{0,P_t} (\lambda) d\lambda \nonumber \\
        =& \int_{span({P}_t)} \exp \left( \frac{ \lambda^\top S_t }{R} - \frac{1}{2} \left\| \lambda \right\|_{\bar{V}_t} \right) f_{0,P_t} (\lambda) \; d \lambda \nonumber \\
        =&
        \int_{span({P}_t)} 
        \exp  \left(  - \frac{1}{2} \left\| \lambda  -  \frac{ \bar{V}_t^\dagger S_t }{R} \right\|_{\bar{V}_t}^2  + \frac{1}{2}
        \left\| \frac{ S_t }{R} \right\|_{\bar{V}_t^\dagger}^2
        \right)  f_{0,P_t} (\lambda) \; d \lambda \label{eq:verify} \\
        =& \exp \left(  \frac{1}{2 R^2 } 
        \left\| S_t \right\|_{\bar{V}_t^\dagger}^2
        \right) \cdot \int_{span( {P}_t)} 
        \exp \left( - \frac{1}{2} \left\| \lambda - \frac{ \bar{V}_t^\dagger S_t }{R} \right\|_{\bar{V}_t}^2 
        \right) f_{0,P_t} (\lambda) \; d\lambda \nonumber \\
        =& \frac{ \exp 
        \left(
        \frac{1}{2 R^2}
        \left\| S_t \right\|_{\bar{V}_t^\dagger}^2
        \right)
        }{ \left( {\det}^* \left(2 \pi {P}_t^\dagger \right) \right)^{1/2} } \cdot \int_{span({P}_t)} 
        \exp \left( - \frac{1}{2} \left\| \lambda - \frac{ \bar{V}_t^\dagger S_t }{R} \right\|_{\bar{V}_t}^2  -  \frac{1}{2} \left\| \lambda \right\|_{ {P}_t}^2 
        \right)  d \lambda \label{eq:before} , 
    \end{align}
    where (\ref{eq:verify}) can be verified by expanding all terms and compare, and (\ref{eq:before}) is from inserting $f_{0,P_t}(\lambda)$ (defined in Eq.~\ref{eq:degerate-density}). 
    
    We also have the following computational identity
    \begin{align}
         \left\| \lambda - \frac{ \bar{V}_t^\dagger S_t }{ R } \right\|_{\bar{V}_t}^2  + \left\| \lambda \right\|_{ P_t}^2 
         &= \left\| \lambda  -  \frac{ {V}_t^\dagger S_t }{R} \right\|_{ {V}_t}^2  + \left\| \frac{ S_t }{R} \right\|_{ \bar{V}_t^\dagger }^2  - \left\|  \frac{ S_t }{R}  \right\|_{ {V}_t ^\dagger }^2, \label{eq:compute-identity}
    \end{align} 
    which can be verified by expanding all terms and using $span (\bar{V}_t) = span (S_t)$. 
    
    We can then use (\ref{eq:compute-identity}) in (\ref{eq:before}) to get
    \begin{align}
        &\frac{ \exp 
        \left(
        \frac{1}{2 R^2}
        \left\| S_t \right\|_{\bar{V}_t^\dagger}^2
        \right)
        }{ \left( {\det}^* \left(2 \pi {P}_t^\dagger \right) \right)^{1/2} } \cdot \int_{span({P}_t)} 
        \exp \left( - \frac{1}{2} \left\| \lambda - \frac{ \bar{V}_t^\dagger S_t }{R} \right\|_{\bar{V}_t}^2  -  \frac{1}{2} \left\| \lambda \right\|_{ {P}_t}^2 
        \right)  d \lambda \nonumber \\
        =&  \frac{ \exp \left( \frac{1}{2 R^2}
        \left\| S_t \right\|_{\bar{V}_t^\dagger}^2
        \right)
        }{ \left( {\det}^* \left(2 \pi {P}_t^\dagger \right) \right)^{1/2} } \cdot \int_{span({P}_t)}  \exp \left( 
         	- \frac{1}{2} \left\| \lambda  - \frac{ {V}_t^\dagger S_t }{R} \right\|_{ {V}_t}^2  - \frac{1}{2} \left\| \frac{ S_t }{R} \right\|_{ \bar{V}_t^\dagger }^2 \hspace{-5pt} + \frac{1}{2} \left\|  \frac{ S_t }{R}  \right\|_{  {V}_t^\dagger }^2   \right)  d \lambda \label{eq:use-identity} \\
        =&
        	\frac{ 
        		\exp \left( \frac{1}{2 R^2} \left\| S_t \right\|_{\bar{V}_t^\dagger}^2 \right)
        	}{ 
        		\left( {\det}^* \left(2 \pi {P}_t^\dagger \right) \right)^{1/2} 
        	} \cdot 
        	\frac{ 
        		\exp \left( \frac{1}{2 R^2} \left\| S_t \right\|_{{V}_t^\dagger}^2 \right)
        	}{ 
        		\exp \left( \frac{1}{2 R^2} \left\| S_t \right\|_{\bar{V}_t^\dagger}^2 \right)
        	}
        \int_{span({P}_t)} 
        \exp  \left(  - \frac{1}{2} 
        \left\| \lambda - \frac{ {V}_t^\dagger S_t }{R} \right\|_{ {V}_t}^2 
         \right)  d\lambda \nonumber \\
        =&
        	\frac{ 
        		\exp \left( \frac{1}{2 R^2} \left\| S_t \right\|_{{V}_t^\dagger}^2 \right)
        	}{ 
        		\left( {\det}^* \left(2 \pi {P}_t^\dagger \right) \right)^{1/2} 
        	} 
        \int_{span({P}_t)} 
        \exp  \left(  - \frac{1}{2} 
        \left\| \lambda - \frac{ {V}_t^\dagger S_t }{R} \right\|_{ {V}_t}^2 
         \right)  d\lambda \nonumber \\
        =& 
        	\frac{ 
        		\exp \left( \frac{1}{2 R^2} \left\| S_t \right\|_{{V}_t^\dagger}^2 \right)
        	}{ 
        		\left( {\det}^* \left(2 \pi {P}_t^\dagger \right) \right)^{1/2} 
        	} 
        	\cdot \left( {\det}^* \left(2 \pi {V}_t^\dagger \right) \right)^{1/2} \nonumber \\
        	&\qquad \cdot \left[ \frac{ 
        		1
        	}{ 
        		\left( {\det}^* \left(2 \pi {V}_t^\dagger \right) \right)^{1/2} 
        	} 
        \int_{span({P}_t)} 
        \exp  \left(  - \frac{1}{2} 
        \left\| \lambda - \frac{ {V}_t^\dagger S_t }{R} \right\|_{ {V}_t}^2 
         \right)  d\lambda \right] \label{eq:transformed-density} \\
        =& 
        \exp \left( \frac{1}{2 R^2 } 
         \left\| S_t \right\|_{ {V}_t ^\dagger }^2
        \right) \frac{ 
        ( {\det}^* (2 \pi {V}_t^\dagger)
        )^{1/2}
        }{ ( {\det}^* (2 \pi {P}_t^\dagger)
        )^{1/2} 
        } , \label{eq:use-density-1}
    \end{align} 
    where (\ref{eq:use-identity}) uses (\ref{eq:compute-identity}), and the last equation is due to $ {V}_t^\dagger S_t \in span ({P}_t) =span ({V}_t) $, and the density of the singular multivariate Gaussian $\mathcal{N} \left( \frac{  {V}_t^\dagger S_t }{R}, {V}_t^\dagger \right) $ (in Eq.~\ref{eq:transformed-density}) integrates to 1 over $span({V}_t)$ (or $span({P}_t)$). 
    
    Next, let $\Lambda$ be the singular multi-variate normal random variable $\Lambda \sim \mathcal{N} \left( \frac{  {V}_t^\dagger S_t }{R}, {V}_t^\dagger \right)$, such that $ \Lambda $ is independent of $\mathcal{F}_\infty$ ($\mathcal{F}_\infty := \sigma (y_1, \eta_1, y_2, \eta_2, y_3, \eta_3, \cdots)$). 
    
    
    Since, by Lemma \ref{lem:prepare-self-normalized}, $ \mathbb{E} \left[ M_t^\lambda \right] \le 1 $ for arbitrary $\lambda$, we have 
    \begin{align}
        &\mathbb{E} \left[  M_t^\Lambda  \right] = \mathbb{E} \left[ \int_{span(\bar{V}_t)} M_t^\lambda f_{0,P_t} (\lambda) d\lambda \right] =
        \mathbb{E} \left[ \exp \left( \frac{1}{2 R^2 } 
         \left\| S_t \right\|_{ {V}_t ^\dagger }^2
        \right) \frac{ 
        ( {\det}^* (2 \pi {V}_t^\dagger)
        )^{1/2}
        }{ ( {\det}^* (2 \pi {P}_t^\dagger)
        )^{1/2} 
        }  \right] 
        \le 1.  \label{eq:EMle1}
    \end{align}
    
    Since ${V}_t^\dagger$ (resp. ${P}_t^\dagger$) is positive semi-definite and symmetric, the pseudo-determinant of ${V}_t^\dagger$ (resp. ${P}_t^\dagger$) is the product of the non-zero eigenvalues of ${V}_t^\dagger$ (resp. ${P}_t^\dagger$). 
    Since rank $rank(V_t) = rank(P_t) = m$, we have 
    \begin{align}
        \frac{ 
        	\left( {\det}^* \left( 2 \pi {V}_t^\dagger \right) \right)^{1/2}
        	}{ \left( {\det}^* \left( 2 \pi {P}_t^\dagger \right) \right)^{1/2} }
        = 
        	\sqrt{ \frac{ {\det}^*  P_t }{ {\det}^* V_t } } \label{eq:concen-inner1}
    \end{align} 
    
    Next, for any $\delta > 0$, we use Markov inequality to get 
    \begin{align*}
        &\quad \mathbb{P} \left(   
	         \left\| S_t \right\|_{ {V}_t^\dagger }^2
	         > 2 R^2  \log \left( \delta^{-1} \sqrt{ \frac{ \det^* V_t }{  \det^* P_t } }\right)
        	\right)  \\
        &= 
        	\mathbb{P} \left(  
		        \exp \left( \frac{1}{2 R^2 } 
		         \left\| S_t \right\|_{ {V}_t ^\dagger }^2
		        \right) \frac{ 
		        \delta (  {\det}^* (2 \pi {V}_t^\dagger)
		        )^{1/2}
		        }{ ( {\det}^* (2 \pi {P}_t^\dagger)
		        )^{1/2} 
		        } > 1
        	\right) \\
        &\le 
        	\mathbb{E} 
        \left[ 
        \exp \left( \frac{1}{2 R^2 } 
         \left\| S_t \right\|_{ {V}_t ^\dagger }^2
        \right) 
        \frac{ 
        \delta ( {\det}^* (2 \pi {V}_t^\dagger) )^{1/2}
        } { 
        ( {\det}^* (2 \pi {P}_t^\dagger)
        )^{1/2}
        } 
        \right] \\
        &= \delta \mathbb{E} 
        \left[ 
        \exp \left( \frac{1}{2 R^2 } 
         \left\| S_t \right\|_{ {V}_t ^\dagger }^2 
        \right) 
        \frac{ 
         ( {\det}^* (2 \pi {V}_t^\dagger) )^{1/2}
        } { 
        ( {\det}^* (2 \pi {P}_t^\dagger)
        )^{1/2}
        } 
        \right] \\
        &\le \delta, 
    \end{align*} 
    where the last inequality is due to (\ref{eq:EMle1}).

\end{proof}


\section{Proof of Lemma \ref{lem:high-prob}}
\label{app:pf-high-prob}

The proof of Lemma \ref{lem:high-prob} needs two auxiliary results, which are in Appendices \ref{app:pf-lem-noise} and \ref{app:prop:compute-bound-on-C2}. 


\begin{lemman}[\ref{lem:high-prob}]
	Under Assumptions  \ref{assumption:bounds}-\ref{assumption:state-eigenvalue}, with probability at least $1 -  4 K \delta$, for $K > K_{\min}$, where 
	\begin{align*}
	    K_{\min} = \Kmin,
	\end{align*}
	the event $\mathcal{E}_{K,\delta} := \left\{ M_* \in \mathcal{C}_* \cap \mathcal{C}_1^{(k)} \cap \mathcal{C}_2^{(k)}, \quad \forall k \in (K_{\min}, K] \right\}$ holds. 
\end{lemman}

\begin{proof}
    To prove Lemma \ref{lem:high-prob}, we need Lemma \ref{lem:projection}, Lemma \ref{lem:noise} and Proposition \ref{prop:compute-bound-on-C2}. Lemma \ref{lem:projection} is proved in Appendix \ref{app:projection}. Lemma \ref{lem:noise} and Proposition \ref{prop:compute-bound-on-C2} are in Appendix \ref{app:pf-lem-noise} and Appendix \ref{app:prop:compute-bound-on-C2} respectively. 
    
    
    \textbf{Step 1: } $M_* \in \mathcal{C}_1^{(k)}$ with high probability.  
    By Lemma \ref{lem:projection} and a union bound, with probability at least $1 - 3K\delta$, 
    \begin{align}
        \left\| P_* - P_k \right\|_2 \le  G_{ k, \delta}, \label{eq:pf-lem3-3}
    \end{align}
    for all $k = K_{\min} + 1,\cdots, K$. 
    
    By triangle inequality, we have
    \begin{align}
    	\left\| ( M_* - M_k ) \left(I - {P}_k^{\mathrm{aug}}\right) \right\|_2 
	    =& \left\| \left( M_* - \bar{X}_k^{\mathrm{next}} \bar{Z}_k^\top V_k^\dagger \right) \left(I  -  {P}_k^{\mathrm{aug}}\right) \right\|_2 \nonumber \\
	    \le & \left\| M_* \left( I  -  {P}_k^{\mathrm{aug}} \right) \right\|_2 
		 +  \left\| \bar{X}_k^{\mathrm{next}} \bar{Z}_k^\top V_k^\dagger \left( I  -  {P}_k^{\mathrm{aug}} \right)  \right\|_2 . \label{eq:projection-part-dummy}
    \end{align}
    Since $V_k (I_{d + d'} - P_k^{\mathrm{aug}}) = 0$, the second term in (\ref{eq:projection-part-dummy}) is zero and we have:
    \begin{align}
	    \left\| ( M_* - M_k ) \left(I - {P}_k^{\mathrm{aug}}\right) \right\|_2 \le& \left\| M_* \left( I -  {P}_k^{\mathrm{aug}} \right) \right\|_2  \overset{ \text{\textcircled{\raisebox{-0.8pt}{1}}} }{=} \left\| M_* \left( P_*^{\mathrm{aug}}  -  {P}_k^{\mathrm{aug}} \right) \right\|_2, \label{eq:in-C-eq2} 
    \end{align} 
    where \textcircled{\raisebox{-0.8pt}{1}} uses Proposition \ref{prop:low-rank-control}. 
    
    Since $\| M_* \| \le 1$ (Assumption \ref{assumption:bounds}), combining (\ref{eq:pf-lem3-3}) and (\ref{eq:in-C-eq2}) gives 
    \begin{align*}
        \left\| ( M_* - M_k ) \left(I - {P}_k^{\mathrm{aug}}\right) \right\|_2 \le \left\| M_* \left( P_*^{\mathrm{aug}}  -  {P}_k^{\mathrm{aug}} \right) \right\|_2 \le G_{k,\delta}. 
    \end{align*}
    
    \textbf{Step 2:} $ M_* \in \mathcal{C}_2^{k}$ with high probability. 
    


    By Proposition \ref{prop:compute-bound-on-C2} (in Appendix \ref{app:prop:compute-bound-on-C2}) and Lemma \ref{lem:noise} (in Appendix \ref{app:pf-lem-noise}), we have, with probability at least $1 - K\delta$, 
    \begin{align}
		&\left\| (M_* - M_k) V_k^{1/2} \right\|_2^2 \nonumber \\
		\le& 4 C^2  \left\| P_* - P_k \right\|_2^2 H k  + 1  + 
    	\left\| \left( P_k W_k \bar{Z}_k^\top \right) \left( V_k^\dagger  \right)^{1/2} \right\|_2^2.  \label{eq:pf-lem3-1} \\
    	\le& 4 C^2  \left\| P_* - P_k \right\|_2^2 H k  + 1  + 
    	2 m^2 C_w^2  \log \left( \delta^{-1}  \left( 1 + \frac{Hk C^2 }{ m }  \right)^{1/2}  \right), \label{eq:pf-lem3-2} \\
    	:=& \beta_{k,\delta}^2 \nonumber
	\end{align}
	where (\ref{eq:pf-lem3-1}) uses Proposition \ref{prop:compute-bound-on-C2} (in Appendix \ref{app:prop:compute-bound-on-C2}) and (\ref{eq:pf-lem3-2}) uses Lemma \ref{lem:noise} (in Appendix \ref{app:pf-lem-noise}). 
	This concludes Step 1.
    
    \textbf{Step 3:} $ M_* \in \mathcal{C}_*$. By item (1) in Assumption \ref{assumption:bounds} and property (\textbf{A}), $ M_* \in \mathcal{C}_*$. 
    
    The above three steps together conclude the proof. 
    
\end{proof}



    



\subsection{} \label{app:pf-lem-noise} 

\begin{lemma} \label{lem:noise} 
	Under Assumption \ref{assumption:bounds}, with probability at least $1 - K \delta$, for all $k = 1,2,3,\cdots, K$, 
	\begin{align*}
		&\left\| \left( P_k W_k \bar{Z}_k^\top \right) \left( V_k^\dagger  \right)^{1/2} \right\|_F^2 
		\le 2 m^2 C_w^2  \log \left( \delta^{-1}  \left( 1 + \frac{Hk C^2 }{ m }  \right)^{1/2}  \right).
	\end{align*}
\end{lemma} 

\begin{proof}
    
	Recall that $P_k = \bar{L}_k \bar{L}_k^\top $. For $j = 1,2,\cdots, m$, let $l_k^{(j)}$ be the $j$-th column of $\bar{L}_k$. 
	Note that 
\begin{align}
    & \left\| \left( P_k W_k \bar{Z}_k^\top \right) \left( V_k^\dagger  \right)^{1/2} \right\|_F^2 \\
    =& tr \left[ P_k W_k \bar{Z}_k^\top V_k^\dagger \bar{Z}_k W_k^\top P_k 
    \right] \nonumber \\
    =& \sum_{j = 1}^m tr \left[ l_k^{(j)} \left( l_k^{(j)} \right)^\top W_k \bar{Z}_k^\top V_k^\dagger \bar{Z}_k W_k^\top l_k^{(j)} \left( l_k^{(j)} \right)^\top 
    \right] \nonumber \\
    =& \sum_{j = 1}^m tr \left[  \left( l_k^{(j)} \right)^\top W_k \bar{Z}_k^\top V_k^\dagger \bar{Z}_k W_k^\top l_k^{(j)}  
    \right] \nonumber \\
    =& \sum_{j = 1}^m 
    \left\| \left( l_k^{(j)} \right)^\top W_k \bar{Z}_k^\top \right\|_{V_k^\dagger}^2 \label{eq:bound-noise-1} .
\end{align}
Since $ \left\| l_k^{(j)}  \right\|_2 = 1$ and $ \left\| w_{k,h} \right\|_2 \le  C_w$, we know that the entries of $ \left( l_k^{(j)} \right) ^\top  W_{k} $ are also $C_w$-sub-Gaussian. 

Then by Lemma \ref{lem:low-rank-concentration} and union bound, with probability at least $1 - K \delta$, for all $j = 1,2, \cdots, d$ and all $k = 1,2,\cdots, K$, 

\begin{align*}
	 \left\| \left( l_k^{(j)} \right)^\top W_k \bar{Z}_k^\top \right\|_{V_k^\dagger}^2 &\le 2 C_w^2  \log \left( \delta^{-1} \sqrt{ \frac{ \det^* V_k }{  \det^* P_k } }\right) \\
	 &  \le 2 C_w^2  \log \left( \delta^{-1} \sqrt{  {\det}^* V_k  } \right) \\
	 &  \le 2 m C_w^2  \log \left( \delta^{-1} \sqrt{  \left(\frac{ tr ( V_k ) }{m} \right)^m } \right) \\
	 &  \le 2 m C_w^2  \log \left( \delta^{-1}  \left( 1 + \frac{Hk C^2 }{ m }  \right)^{1/2}  \right), 
\end{align*} 
where the second last inequality uses the AM-GM inequality, and that $V_k$ is PSD and has $m$ non-zero eigenvalues, and the last inequality uses $\left\| \hat{z}_{k,h} \right\|_2 \le C$ (Proposition \ref{prop:stable}).  

Plugging the above result give into (\ref{eq:bound-noise-1}) gives: with probability at least $1 - K\delta$, 
\begin{align*}
    \left\| \left( P_k W_k \bar{Z}_k^\top \right) \left( V_k^\dagger  \right)^{1/2} \right\|_F^2 &= \sum_{j = 1}^m 
    \left\| \left( l_k^{(j)} \right)^\top W_k \bar{Z}_k^\top \right\|_{V_k^\dagger}^2 \\
    &\le 2 m^2 C_w^2  \log \left( \delta^{-1}  \left( 1 + \frac{ Hk C^2 }{ m }  \right)^{1/2}  \right),
\end{align*}
for all $k = 1,2,3\cdots, K$.
\end{proof}

\subsection{}
\label{app:prop:compute-bound-on-C2}

\begin{proposition} \label{prop:compute-bound-on-C2}
	Under Assumptions \ref{assumption:bounds} and \ref{assumption:low-dimension}, we have 
	\begin{align*}
		&  \left\| (M_* - M_k) V_k^{1/2} \right\|_2^2 \le 4 C^2  \left\| P_* - P_k \right\|_2^2 H k  +  1  + 
    	\left\| \left( P_k W_k \bar{Z}_k^\top \right) \left( V_k^\dagger  \right)^{1/2} \right\|_2^2. 
	\end{align*}
\end{proposition} 


\begin{proof}
	Using 
	\begin{align*}
		M_k^\top &= \Mk \\
		\hat{X}_k^{\mathrm{next}} &= M_* \hat{Z}_k + W_k, \\
		\bar{X}_k^{\mathrm{next}} &= P_k \hat{X}_k^{\mathrm{next}}
	\end{align*}
	we have 
	\begin{align*}
	    M_k = \bar{X}_k^{\mathrm{next}} \bar{Z}_k^\top V_k^\dagger = P_k (M_* \hat{Z}_k + W_k) \bar{Z}_k^\top V_k^\dagger.
	\end{align*}
	Thus, 
	\begin{align*}
	&(M_* - M_k) V_k  (M_* - M_k)^\top \\
	=& 
	\left(M_* - P_k ( M_* \hat{Z}_k + W_k) \bar{Z}_k^\top V_k^\dagger \right) V_k  \left( M_* - P_k( M_* \hat{Z}_k + W_k) \bar{Z}_k^\top V_k^\dagger \right)^\top \\
	\overset{\text{\textcircled{1}}}{=}& \left( M_* V_k - P_k M_* \hat{Z}_k \bar{Z}_k^\top  - P_k W_k \bar{Z}_k^\top  \right) V_k^\dagger  \left( M_* V_k - P_k M_* \hat{Z}_k \bar{Z}_k
	^\top  - P_k W_k \bar{Z}_k^\top  \right)^\top ,
\end{align*}
where \textcircled{1} can be verified by expanding all terms and comparing. 

    
We can expand the above and get
\begin{align}
    &\left\| \left(M_* - M_k \right) V_k^{1/2} \right\|_2^2 \nonumber \\
	=& \left\| \left(M_* - M_k \right) V_k \left( M_* - M_k \right)^\top \right\|_2 \nonumber \\
	=&  \Bigg\| \left(M_* - P_k ( M_* \hat{Z}_k + W_k) \bar{Z}_k^\top V_k^\dagger \right) V_k   \left( M_* - P_k( M_* \hat{Z}_k + W_k) \bar{Z}_k^\top V_k^\dagger \right)^\top \Bigg\|_2 \nonumber \\
	=& 
	 \bigg\|  \left( M_* \left( \bar{Z}_k \bar{Z}_k^\top  +  P_k^{\mathrm{aug}} \right)  -  P_k M_* \hat{Z}_k \bar{Z}_k^\top  -  P_k W_k \bar{Z}_k^\top  \right)   V_k^\dagger \nonumber \\
	 &\qquad \qquad \cdot \left( M_*  \left( \bar{Z}_k \bar{Z}_k^\top   +  P_k^{\mathrm{aug}} \right)  -  P_k M_* \hat{Z}_k \bar{Z}_k
	^\top  - P_k W_k \bar{Z}_k^\top  \right)^\top  \bigg\|_2  
	\nonumber \\
	\overset{\textcircled{1}}{\le}& \left\| \left( M_*  \bar{Z}_k \bar{Z}_k^\top - P_k M_* \hat{Z}_k \bar{Z}_k^\top  \right) V_k^\dagger \left( M_*  \bar{Z}_k \bar{Z}_k^\top - P_k M_* \hat{Z}_k \bar{Z}_k^\top  \right)^\top \right\|_2 \nonumber \\
	&\qquad +  \left\| \left( M_* P_k^{\mathrm{aug}} \right) V_k^\dagger \left( M_* P_k^{\mathrm{aug}} \right)^\top \right\|_2 +  \left\| \left( P_k W_k \bar{Z}_k^\top \right) V_k^\dagger \left( P_k W_k \bar{Z}_k^\top \right)^\top \right\|_2
	\nonumber \\
	{=}& \left\| \left( M_*  \bar{Z}_k \bar{Z}_k^\top - P_k M_* \hat{Z}_k \bar{Z}_k^\top  \right) \left( V_k^\dagger \right)^{1/2} \right\|_2^2 \label{eq:in-app-regroup-1} \\
	&\qquad +  \left\| \left( M_* P_k^{\mathrm{aug}} \right) \left( V_k^\dagger \right)^{1/2} \right\|_2^2 +  \left\| \left( P_k W_k \bar{Z}_k^\top \right) \left( V_k^\dagger  \right)^{1/2} \right\|_2^2, \label{eq:in-app-regroup-2}
\end{align}
where \textcircled{1} uses a triangle inequality. 

For the term in (\ref{eq:in-app-regroup-1}), using $ M_* = M_* P_*^{\mathrm{aug}} $ (Proposition \ref{prop:low-rank-control}), and $ \bar{Z}_k = P_k^{\mathrm{aug}} \hat{Z}_k$, we have 
\begin{align}
    \left\| \left( M_*  \bar{Z}_k \bar{Z}_k^\top - P_k M_* \hat{Z}_k \bar{Z}_k^\top  \right) \left( V_k^\dagger \right)^{1/2} \right\|_2^2
    &= 
    \left\| \left( M_* P_k^{\mathrm{aug}} \hat{Z}_k \bar{Z}_k^\top - P_k M_* \hat{Z}_k \bar{Z}_k^\top  \right) \left( V_k^\dagger \right)^{1/2} \right\|_2^2 \nonumber \\
    &\le 
    \left\| \left( M_* P_k^{\mathrm{aug}} - P_k M_* \right) \hat{Z}_k \bar{Z}_k^\top \left( V_k^\dagger \right)^{1/2} \right\|_2^2 \nonumber \\
    &\le 
    \left\| \left( M_* P_k^{\mathrm{aug}} - P_k M_* \right) \right\|_2 \left\| \hat{Z}_k \bar{Z}_k^\top \left( V_k^\dagger \right)^{1/2} \right\|_2^2 \label{eq:in-app-for-regroup-term1}
\end{align}

Since $ V_k = \bar{ Z }_k \bar{Z}_k + P_k^{\mathrm{aug}} $, we have $ \left\| \bar{ Z }_k \left( V_k^\dagger \right)^{1/2} \right\| \le 1$. Also,
\begin{align*}
	\left\| \hat{Z}_k \right\|_2^2 = \left\| \hat{Z}_k \hat{Z}_k^\top \right\|_2 = \left\| \sum_{k'=1}^k \sum_{h=1}^H \hat{z}_{k',h} \hat{z}_{k',h}^\top  \right\|_2 \le C^2 {Hk}, 
\end{align*}
where the last inequality uses a triangle inequality and $\left\| \hat{z}_{k',h}^\top  \right\|_2 \le C$ (Proposition \ref{prop:stable}). 


Under Assumption \ref{assumption:bounds}, $\| M_* \|_2 \le 1$. Thus we have
$
    \left\|  M_* P_k^{\mathrm{aug}} \left( V_k^\dagger \right)^{1/2} \right\|_2^2 \le 1. 
$

Collecting relevant terms, we can continue from (\ref{eq:in-app-regroup-2}) to get 
\begin{align}
    \left\| \left(M_* - M_k \right) V_k^{1/2} \right\|_2^2
	\le&  C^2  \left\|  M_* P_k^{\mathrm{aug}}  -  P_k M_*  \right\|_2^2  Hk  +  1  +  \left\|  P_k W_k \bar{Z}_k^\top  \left( V_k^\dagger  \right)^{1/2} \right\|_2^2 . \label{eq:within-lemma-con} 
\end{align}

Then we use $ \| M_* \|_2 \le 1 $ (Assumption \ref{assumption:bounds}) and Proposition \ref{prop:low-rank-control} to get 
\begin{align*}
	  & \left\| M_* P_k^{\mathrm{aug}}  -  P_k M_* \right\|_2 \\
	 =& \left\| P_* M_* P_k^{\mathrm{aug}}  -  P_* M_* P_*^{\mathrm{aug}}  +  P_* M_* P_*^{\mathrm{aug}}  -  P_k M_* P_*^{\mathrm{aug}} \right\|_2 \\
	 =& \left\| P_* M_* ( P_k^{\mathrm{aug}}  -  P_*^{\mathrm{aug}} )  +  ( P_* - P_k ) M_* P_*^{\mathrm{aug}}  \right\|_2 \\
	 \le & 2 \left\|  P_k  -  P_*  \right\|_2.
\end{align*}
Collecting terms yields the expression in the lemma statement.
\end{proof}


\section{Proof of Proposition \ref{prop:decomp} }
\label{app:pf-prop-decomp}

\begin{propositionn}
    \label{prop:decomp-in-app} 
    Let $ \widetilde{\Psi}_{k,h} := \Psi_h (\widetilde{M}_k) $ computed by (\ref{eq:def-Ph}). 
    Under event $\mathcal{E}_{K,\delta}$ ($K > K_{\min}^2$), we have 
    \begin{align}
	    \mathrm{Reg}(K) \le \mathcal{O} \left( H \sqrt{K} \right) + \sum_{k=\left\lceil \sqrt{K} \right\rceil+ 1}^K \sum_{h=1}^{H-1}  \left( \Delta_{k,h} + \Delta_{k,h}^\prime + \Delta_{k,h}^{\prime \prime } \right),  \label{eq:regret-decomp-in-app}
    \end{align}    
    where $\Delta_{k,h}  := \mathbb{E} \left[ J_{h+1}^{\pi_k } \left( M_*, \hat{x}_{k,h+1} \right) | \mathcal{F}_{k,h} \right] -  J_{h+1}^{\pi_k } \left( M_*, \hat{x}_{k,h+1} \right)$, 
    $\Delta_{k,h}' := \left\| \hat{x}_{k,h+1} \right\|_{ \widetilde{\Psi}_{k,h+1} }  - \mathbb{E} \left[  \left\| \hat{x}_{k,h+1} \right\|_{ \widetilde{\Psi}_{k,h+1} } | \mathcal{F}_{k,h} \right] $, 
    $\Delta_{k,h}'' := \left\|  M_* \hat{z}_{k,h} \right\|_{ \widetilde{\Psi}_{k,h+1} }   - \left\| \widetilde{M}_k \hat{z}_{k,h} \right\|_{ \widetilde{\Psi}_{k,h+1} } $, and $\mathcal{F}_{k,h}$ is all randomness before time $(k,h)$.
\end{propositionn}

\begin{proof}


    

By boundedness results in Proposition \ref{prop:stable}, we have the costs $c_{k,h}$ satisfies
\begin{align}
    c_{k,h } = \hat{x}_{k,h}^\top Q_h \hat{x}_{k,h} + {u}_{k,h}^\top R_h {u}_{k,h} \le 2 C . \label{eq:}
\end{align}
Thus by (\ref{eq:}), for the first $ \left\lceil \sqrt{K} \right\rceil $ episodes, we have 
\begin{align}
    Reg \left( \left\lceil \sqrt{K} \right\rceil \right) \le \sum_{k=1}^{ \left\lceil \sqrt{K} \right\rceil } J_h^{ \pi_k } \left( {M}_*, \hat{x}_{k,h}   \right) \le \mathcal{O } \left( H \sqrt{K} \right). \label{eq:regret-sqrtK}
\end{align}

Under the event $\mathcal{E}_{K, \delta}$, we have the optimistic rule (for $k > K_{\min}$): 
\begin{align}
	 J_1^{* } \left( \widetilde{M}_k, \hat{x}_{k,1} \right) \le  J_1^{*} \left( M_*, \hat{x}_{k,1} \right). \label{eq:opt-prop-in-app}
\end{align}

Write
$ \Theta_{k,h} := J_h^{ \pi_k } \left( {M}_*, \hat{x}_{k,h} \right) - J_h^{ * } \left( \widetilde{M}_k, \hat{x}_{k,h} \right). $ 
Then under event $\mathcal{E}_{K,\delta}$, since $K > K_{\min}^2$, the regret (Eq. \ref{eq:def-regret}) can be bounded by: 
\begin{align} 
	\mathrm{Reg}(K) &= \sum_{k=1}^K J_1^{\pi_k } \left( M_*, \hat{x}_{k,1} \right) - J_1^{*} \left( M_*, \hat{x}_{k,1} \right) \nonumber \\
	&= \sum_{k=1}^{ \left\lceil \sqrt{K} \right\rceil } \left[ J_1^{\pi_k } \left( M_*, \hat{x}_{k,1} \right) - J_1^{*} \left( M_*, \hat{x}_{k,1} \right) \right] + \sum_{ k = \left\lceil K \right\rceil + 1 }^K \left[ J_1^{\pi_k } \left( M_*, \hat{x}_{k,1} \right) - J_1^{*} \left( M_*, \hat{x}_{k,1} \right) \right] \nonumber \\
	& \overset{\textcircled{1}}{\le} \mathcal{O} \left( \sqrt{HK} \right) + \sum_{k = \left\lceil K \right\rceil + 1}^K 
	J_1^{ \pi_k } \left( {M}_*, \hat{x}_{k,h} \right) - J_1^{ * } \left( {M}_*, \hat{x}_{k,h} \right) \nonumber \\
	& \overset{\textcircled{2}}{\le} \mathcal{O} \left( \sqrt{HK} \right) + \sum_{k = \left\lceil K \right\rceil + 1}^K 
    \Theta_{k,1}, \label{eq:in-app-regret-decomp-1}
\end{align}
where \textcircled{1} uses (\ref{eq:regret-sqrtK}) and \textcircled{2} uses the optimism property (\ref{eq:opt-prop-in-app}). 

Let $\mathcal{F}_{k,h}$ be the $\sigma$-algebra generated by all randomness up to $ (k,h)$. 

\textit{\textbf{Note.} Next we focus on $k \ge K_{\min} + 1 $. }

For simplicity, let $\widetilde{\Psi}_{k,h} := \Psi_h (\widetilde{M}_k)$ and $\widetilde{\psi}_{k,h} $ be the $\Psi_h$ and $\psi_h$ quantities computed with $\widetilde{M}_k$ using (\ref{eq:def-Ph}). 
Using (\ref{eq:def-cost-c}) and (\ref{eq:def-J-via-c}), we can compute the term $\Delta_{k,h}$ as follows. 
\begin{align}
	\Theta_{k,h}  =&  \left\| \hat{x}_{k,h} \right\|_{ Q_h}  +  \left\| {a}_{k,h} \right\|_{ R_h}  \nonumber \\
	&+  \mathbb{E} \left[ J_{h+1}^{\pi_k } \left( M_*, \hat{x}_{k,h+1} \right) | \mathcal{F}_{k,h} \right] \nonumber \\
	&- \left\| \hat{x}_{k,h} \right\|_{ Q_h } - \left\| {a}_{k,h} \right\|_{R_h}  \nonumber \\
	&- \mathbb{E} \left[ \hat{x}_{k,h+1}^\top \widetilde{\Psi}_{k,h+1}  \hat{x}_{k,h+1} \big| \mathcal{F}_{k,h} \right] -  \widetilde{\psi}_{k,h+1}   \nonumber,
\end{align}
which gives,
\begin{align}
	 \Theta_{k,h} =& \mathbb{E} \left[ J_{h+1}^{\pi_k } \left( M_*, \hat{x}_{k,h+1} \right) | \mathcal{F}_{k,h} \right] - \mathbb{E} \left[ \left\| \widetilde{M}_k \hat{z}_{k,h} + w_{k,h+1} \right\|_{ \widetilde{\Psi}_{k,h+1}  }  \bigg| \mathcal{F}_{k,h} \right] -  \widetilde{\psi}_{k,h+1} \nonumber \\
	=&  \Delta_{k,h} + J_{h+1}^{\pi_k } \left( M_*, \hat{x}_{k,h+1} \right) - \mathbb{E} \left[ \left\| \widetilde{M}_k \hat{z}_{k,h} + w_{k,h+1} \right\|_{ \widetilde{\Psi}_{k,h+1}  }  \bigg| \mathcal{F}_{k,h} \right]  -  \widetilde{\psi}_{k,h+1}, \label{eq:regret-decompose-1}
\end{align}
where 
\begin{align*}
\Delta_{k,h}  :=& \mathbb{E} \left[ J_{h+1}^{\pi_k } \left( M_*, \hat{x}_{k,h+1} \right) | \mathcal{F}_{k,h} \right]  -  J_{h+1}^{\pi_k } \left( M_*, \hat{x}_{k,h+1} \right). 
\end{align*}
Since noise is independent and mean zero (Assumption \ref{assumption:noise}), we have $\mathbb{E} \left[ ( \widetilde{M}_k \hat{z}_{k,h} )^\top  \widetilde{\Psi}_{k,h+1} w_{k,h} | \mathcal{F}_{k,h} \right] = 0$.
We then have 
\begin{align}
	\Theta_{k,h} =& \Delta_{k,h} + J_{h+1}^{\pi_k } \left( M_*, \hat{x}_{k,h+1} \right) - \left\| \widetilde{M}_k \hat{z}_{k,h} \right\|_{\widetilde{\Psi}_{k,h+1}} \nonumber \\
	&- \mathbb{E} \left[ \left\| w_{k,h+1} \right\|_{ \widetilde{\Psi}_{k,h+1} }   | \mathcal{F}_{k,h} \right] \hspace{-3pt} -  \widetilde{\psi}_{k,h+1} \nonumber \\
	  =& \Delta_{k,h} + J_{h+1}^{\pi_k } \left( M_*, \hat{x}_{k,h+1} \right) 
	 - \left\| \widetilde{M}_k \hat{z}_{k,h} \right\|_{ \widetilde{\Psi}_{k,h+1} }  - \widetilde{\psi}_{k,h+1} \nonumber \\
	 &- \mathbb{E} \left[ \left\| \hat{x}_{k,h+1} - M_* \hat{z}_{k,h} \right\|_{ \widetilde{\Psi}_{k,h+1} }  | \mathcal{F}_{k,h} \right] \nonumber \\
	  =& \Delta_{k,h} + J_{h+1}^{\pi_k } \left( M_*, \hat{x}_{k,h+1} \right) 
	 - \left\| \widetilde{M}_k \hat{z}_{k,h} \right\|_{ \widetilde{\Psi}_{k,h+1} }  - \widetilde{\psi}_{k,h+1} \nonumber \\
	 &- \mathbb{E} \left[ \left\| \hat{x}_{k,h+1} \right\|_{ \widetilde{\Psi}_{k,h+1} }  | \mathcal{F}_{k,h} \right]  +  \left\|  M_* \hat{z}_{k,h} \right\|_{ \widetilde{\Psi}_{k,h+1} } .  \label{eq:regret-decompose-2}
\end{align}
From (\ref{eq:J-from-psi}), we know 
\begin{align*}
    J_h^{ * } \left( \widetilde{M}_k, \hat{x}_{k,h+1} \right) = \left\| \hat{x}_{k,h+1} \right\|_{ \widetilde{\Psi}_{k,h+1}  } + \widetilde{\psi}_{k,h+1}.
\end{align*}
Thus we can rewrite (\ref{eq:regret-decompose-2}) into 
\begin{align}
	 \Theta_{k,h} =&\Delta_{k,h} + J_{h+1}^{\pi_k } \left( M_*, \hat{x}_{k,h+1} \right) 
	 - J_{h+1}^{ * } \left( \widetilde{M}_k, \hat{x}_{k,h+1} \right)  \nonumber \\
	 &+ \left\| \hat{x}_{k,h+1} \right\|_{ \widetilde{\Psi}_{k,h+1} }    - \mathbb{E} \left[  \left\| \hat{x}_{k,h+1} \right\|_{ \widetilde{\Psi}_{k,h+1} }   | \mathcal{F}_{k,h} \right] \nonumber \\
	 &+ \left\|  M_* \hat{z}_{k,h} \right\|_{ \widetilde{\Psi}_{k,h+1} }   - \left\|  \widetilde{M}_k \hat{z}_{k,h} \right\|_{ \widetilde{\Psi}_{k,h+1} } \nonumber \\
	 =& \Delta_{k,h} + \Delta_{k,h}^\prime + \Delta_{k,h}^{\prime \prime} + \Theta_{k,h+1}, \nonumber
\end{align}
where 
\begin{align*}
	\Delta_{k,h}^\prime :=& \left\| \hat{x}_{k,h+1} \right\|_{ \widetilde{\Psi}_{k,h+1} }  - \mathbb{E} \left[  \left\| \hat{x}_{k,h+1} \right\|_{ \widetilde{\Psi}_{k,h+1} } | \mathcal{F}_{k,h} \right] , \\
	\Delta_{k,h}^{\prime \prime} :=& \left\|  M_* \hat{z}_{k,h} \right\|_{ \widetilde{\Psi}_{k,h+1} }   - \left\| \widetilde{M}_k \hat{z}_{k,h} \right\|_{ \widetilde{\Psi}_{k,h+1} } 
\end{align*}


Since the cost for $H+1$ and beyond are always zero, we know that the regret can be bounded as:
\begin{align*}
	\sum_{k=\left\lceil \sqrt{K} \right\rceil+ 1}^K  \Theta_{k,1} \le \sum_{k=\left\lceil \sqrt{K} \right\rceil+ 1}^K \sum_{h=1}^H   \left( \Delta_{k,h} + \Delta_{k,h}^\prime + \Delta_{k,h}^{\prime \prime } \right). 
\end{align*}

Plug the above expression into (\ref{eq:in-app-regret-decomp-1}) concludes the proof.

\end{proof}

\section{Proof of Lemma \ref{lem:regret-simple-terms}}
\label{app:pf-regret-simple-terms}

\begin{lemman}[\ref{lem:regret-simple-terms}]
	Under Assumptions \ref{assumption:bounds}-\ref{assumption:state-eigenvalue}, conditioning on $\mathcal{E}_{K,\delta}$ being true, with probability at least $ 1 - 2 \delta $, we have both 
	\begin{align*} 
		&\left| \sum_{k=\left\lceil \sqrt{K} \right\rceil+ 1}^K \sum_{h=1}^{H-1} \Delta_{k,h} \right|	 \hspace{-3pt} \le \mathcal{O} \left( \sqrt{ K H^3 \log \frac{2}{\delta} } \right) \quad  \text{and} \quad \left| \sum_{k=\left\lceil \sqrt{K} \right\rceil+ 1}^K \sum_{h=1}^{H-1} \Delta_{k,h}^\prime \right|	\le \mathcal{O} \left( \sqrt{  H K  \log \frac{2}{\delta} } \right) . 
	\end{align*} 
\end{lemman}

\begin{proof}
    Fistly, we show that $ | \Delta_{k,h} |$ and $| \Delta_{k,h}' |$ are bounded. 
    
    Recall $ \widetilde{\Psi}_{k,h+1} : = \Psi_h (\widetilde{M}_k) $ (computed from Eq. \ref{eq:def-Ph}). 
    From Assumption \ref{assumption:bounds}, Propositions \ref{prop:Psi-bounded} and \ref{prop:stable}, we know that \begin{align*}
        |\Delta_{k,h}^\prime | = \left| \left\| \hat{x}_{k,h+1} \right\|_{ \widetilde{\Psi}_{k,h+1} }  - \mathbb{E} \left[  \left\| \hat{x}_{k,h+1} \right\|_{ \widetilde{\Psi}_{k,h+1} } | \mathcal{F}_{k,h} \right] \right| \le 2 C .
    \end{align*} 
    
    For the bound of $|\Delta_{k,h}|$, we use an induction argument to bound it. 
    By Proposition \ref{prop:stable} and $\| Q_h \|_2 \le C$ (Assumption \ref{assumption:bounds}), we have 
	\begin{align}
		\left| J_H^{\pi_k} (M_*, \hat{x}_{k,H}) \right| = \hat{x}_{k,H}^\top Q_H  \hat{x}_{k,H} \le C. 
	\end{align}
	Inductively, 
	\begin{align*}
		\left| J_h^{\pi_k } (M_*, \hat{x}_{k,h}) \right| 
		\le& \left| \hat{x}_{k,h}^\top Q_h \hat{x}_{k,h} + u_{k,h}^\top R_h u_{k,h} \right| + \left| \mathbb{E} \left[ J_{h+1}^{\pi_k } \left( M_*, \hat{x}_{k,h+1} \right) | \mathcal{F}_{k,h} \right] \right| \\
		\le& \hat{z}_{k,h}^\top \begin{bmatrix} Q_h & 0 \\ 0 & R_h \end{bmatrix} \hat{z}_{k,h} + (H-h)C \\
		\le&  C^3 + (H-h)C \\
		=& (H-h+ C^2 )C 
	\end{align*}
	Next, since the costs at $H+1$ are always zero, 
	\begin{align*}
		\left| \Delta_{k,h} \right| = \left| \mathbb{E} \left[ J_{h+1}^{\pi_k } \left( M_*, \hat{x}_{k,h+1} \right) | \mathcal{F}_{k,h} \right]  -  J_{h+1}^{\pi_k } \left( M_*, \hat{x}_{k,h+1} \right) \right| 
		\le 2 \left( H-h+ C^2 \right) C. 
	\end{align*}
	
	Let $\mathcal{F}_{k,h}$ be all randomness before time $(k,h)$. Then, 
	\begin{align*}
	    &\mathbb{E} \left[ \Delta_{k,h} | \mathcal{F}_{k,h} \right] =  \mathbb{E} \left[ J_{h+1}^{\pi_k } \left( M_*, \hat{x}_{k,h+1} \right) | \mathcal{F}_{k,h} \right]  -  \mathbb{E} \left[ J_{h+1}^{\pi_k } \left( M_*, \hat{x}_{k,h+1} \right)  | \mathcal{F}_{k,h} \right] = 0, \\
	    &\mathbb{E} \left[ \Delta_{k,h}' | \mathcal{F}_{k,h} \right] =  \mathbb{E} \left[\left\| \hat{x}_{k,h+1} \right\|_{ \widetilde{\Psi}_{k,h+1} } | \mathcal{F}_{k,h} \right] - \mathbb{E} \left[  \left\| \hat{x}_{k,h+1} \right\|_{ \widetilde{\Psi}_{k,h+1} } | \mathcal{F}_{k,h} \right] = 0.
	\end{align*}
	
	Thus $\{ \Delta_{k,h} \} $ and $\{ \Delta_{k,h} \}$ are two bounded martingale difference sequence. We can then apply the Azuma's inequality to get the lemma statement. 
\end{proof}

\section{Proof of Lemma \ref{lem:bound-delta-pp}}
\label{app:pf-bound-delta-pp}


\begin{lemman}[\ref{lem:bound-delta-pp}]
	Let Assumptions \ref{assumption:bounds} and \ref{assumption:low-dimension} hold. Under event $\mathcal{E}_{K, \delta}$ ($K > K_{\min}$),  we have 
    \begin{align*}
	    \sum_{k = K_{\min} + 1}^K \sum_{h=1}^{H-1} \Delta_{k,h}'' \le  \widetilde{\mathcal{O}} \left( \left( H^{5/2} + m^{3/2} H \right) \sqrt{K} \right).
    \end{align*} 
\end{lemman}

Lemma \ref{lem:bound-delta-pp} is direct result of Lemma \ref{lem:bound-delta-pp-1} and Lemma \ref{lem:bound-delta-pp-2}, which are presented below in Appendices \ref{app:bound-delta-pp-1} and \ref{app:bound-delta-pp-2} respectively. 

\subsection{} \label{app:bound-delta-pp-1}
\begin{lemma} 
    \label{lem:bound-delta-pp-1}
	Let Assumptions \ref{assumption:bounds} and  \ref{assumption:low-dimension} hold. Under event $\mathcal{E}_{K,\delta}$ ($K > K_{\min}$), we have 
	\begin{align*}
			&\sum_{k=\left\lceil \sqrt{K} \right\rceil+ 1}^K \sum_{h=1}^H \Delta_{k,h}^{\prime \prime} 	
			\le \widetilde{ \mathcal{O} } \left( \left( H + m \right) \Gamma_{K}  \sqrt{ HK } \right) , 
	\end{align*} 
	where 
	\begin{align*}
		\Gamma_{K} \hspace{-3pt} &:= \hspace{-3pt} \left[  \sum_{k= \sqrt{K} }^K \sum_{h=1}^H   \left\| \left( V_k^\dagger \right)^{1/2}  \hat{z}_{k,h} \right\|_2^2   \right]^{1/2} \hspace{-0pt}  . 
	\end{align*}
\end{lemma}

\begin{proof}


We compute 
\begin{align}
	&\sum_{k=\sqrt{K}}^K \sum_{h=1}^{H-1} \Delta_{k,h}^{\prime \prime} 
	\le \sum_{k=\sqrt{K}}^K \sum_{h=1}^{H-1} \left| \Delta_{k,h}^{\prime \prime} \right| \nonumber \\
	=&  \sum_{k=\sqrt{K}}^K \sum_{h=1}^{H-1}  \left| \left\|  \widetilde{\Psi}_{k,h+1}^{1/2}   M_* \hat{z}_{k,h}  \right\|_2^2  - \left\|  \widetilde{\Psi}_{k,h+1}^{1/2}   \widetilde{M}_k \hat{z}_{k,h}  \right\|_2^2 \right| \nonumber \\
	=& \sum_{k=\sqrt{K}}^K \sum_{h=1}^{H-1}  \left|  \left\|  \widetilde{\Psi}_{k,h+1}^{1/2}   M_* \hat{z}_{k,h}  \right\|_2 - \left\|  \widetilde{\Psi}_{k,h+1}^{1/2}  \widetilde{M}_k \hat{z}_{k,h}  \right\|_2  \right| \cdot \left| \left\|  \widetilde{\Psi}_{k,h+1}^{1/2}   M_* \hat{z}_{k,h}  \right\|_2 +  \left\| \widetilde{\Psi}_{k,h+1}^{1/2}  \widetilde{M}_k \hat{z}_{k,h} \right\|_2 \right| \nonumber \\
	 \le& 
		\left[ \sum_{k=\sqrt{K}}^K \sum_{h=1}^{H-1}  \left(  \left\|  \widetilde{\Psi}_{k,h+1}^{1/2}   M_* \hat{z}_{k,h}  \right\|_2 -  \left\|  \widetilde{\Psi}_{k,h+1}^{1/2}  \widetilde{M}_k \hat{z}_{k,h}  \right\|_2  \right)^2 \right]^{1/2}  \\
		&\qquad \cdot  \left[ \sum_{k=\sqrt{K}}^K \sum_{h=1}^{H-1}  \left(  \left\|  \widetilde{\Psi}_{k,h+1}^{1/2}  M_* \hat{z}_{k,h}  \right\|_2 + \left\|  \widetilde{\Psi}_{k,h+1}^{1/2}  \widetilde{M}_k \hat{z}_{k,h}  \right\|_2 \right)^2  \right]^{1/2}, \label{eq:app-bound-delta-pp-part1}
\end{align}
where in the last step we use the Cauchy-Schwarz inequality. Under Assumption \ref{assumption:bounds}, we have 
\begin{align} 
	 &\left[ \sum_{k=\sqrt{K}}^K \sum_{h=1}^{H-1}  \left(  \left\|  \widetilde{\Psi}_{k,h+1}^{1/2}   M_* \hat{z}_{k,h}  \right\|_2  +  \left\| \widetilde{\Psi}_{k,h+1}^{1/2}  \widetilde{M}_k \hat{z}_{k,h} \right\|_2  \right)^2 \right]^{1/2} \le 2 C^3 \sqrt{  H K }. \label{eq:app-bound-delta-pp-part2} 
\end{align} 

Next, under event $\mathcal{E}_{K,\delta}$ ($K^2 > K_{\min}$), we have 
\begin{align}
	 & \left| \left\| \widetilde{\Psi}_{k,h+1}^{1/2} \left(  M_* \hat{z}_{k,h} \right) \right\|_2 - \left\| \widetilde{\Psi}_{k,h+1}^{1/2} \left(  \widetilde{M}_k \hat{z}_{k,h} \right) \right\|_2  \right|^2 \nonumber \\ 
	 \le& \left\| \widetilde{\Psi}_{k,h+1}^{1/2} \left( \widetilde{M}_k  - {M}_*  \right) \hat{z}_{k,h} \right\|_2^2 \nonumber \\
	 \overset{\textcircled{1}}{\lesssim}& \left\| \left( \widetilde{M}_k  - {M}_*  \right) \hat{z}_{k,h} \right\|_2^2 \nonumber \\
	 \le& \left\| \left( \widetilde{M}_k  - {M}_*  \right)  P_k^{\mathrm{aug}} \hat{z}_{k,h} \right\|_2^2 + \left\| \left( \widetilde{M}_k  - {M}_*  \right) (I - P_k^{\mathrm{aug}}) \hat{z}_{k,h} \right\|_2^2 , \label{eq:app-delta-pp-dummy-1}
\end{align}
where \textcircled{1} uses boundedness of $\widetilde{\Psi}_{k,h+1}$ (Proposition \ref{prop:Psi-bounded}). 

Under event $\mathcal{E}_{K,\delta}$, the second term in (\ref{eq:app-delta-pp-dummy-1}) can be taken care of by the confidence region $\mathcal{C}_2^{(k)}$: for $k > K_{\min}$, 
\begin{align}
	 &\left\| \left( \widetilde{M}_k  - {M}_*  \right) (I - P_k^{\mathrm{aug}}) \hat{z}_{k,h} \right\|_2^2 = \mathcal{O} \left(  G_{k,\delta}^2 \right)  = \mathcal{O} \left( \frac{ H }{k} \log (d/ \delta) \right), \label{eq:app-projection-dummy-1}
\end{align}
which sums to $\widetilde{\mathcal{O}}( H^2 )$ over $k$ and $h$.

We now focus on the first term in (\ref{eq:app-delta-pp-dummy-1}). For this term, we have, under event $ \mathcal{E}_{K,\delta}$, 
\begin{align}
	\left\| \left( \widetilde{M}_k  - {M}_*  \right)  P_k^{\mathrm{aug}} \hat{z}_{k,h} \right\|_2^2 =& \left\| \left( \widetilde{M}_k  - {M}_*  \right)  V_k^{1/2} \left( V_k^\dagger \right)^{1/2} P_k^{\mathrm{aug}} \hat{z}_{k,h} \right\|_2^2 \nonumber \\
	\le&  \left\| \left( \widetilde{M}_k  - {M}_*  \right)  V_k^{1/2}  \right\|_2^2 \left\|\left( V_k^\dagger \right)^{1/2} P_k^{\mathrm{aug}} \hat{z}_{k,h} \right\|_2^2 \nonumber \\ 
     \le& 
     	2 \beta_{k,\delta} \left\| \left( V_k^\dagger \right)^{1/2} P_k^{\mathrm{aug}} \hat{z}_{k,h} \right\|_2^2 \label{eq:use-conf-C1} \\
     \le& 
     	2 \beta_{k,\delta} \left\| \left( V_k^\dagger \right)^{1/2} \hat{z}_{k,h} \right\|_2^2, \label{eq:use-projection-prop} 
\end{align}
where in (\ref{eq:use-conf-C1}) we bound $\left\| \left( \widetilde{M}_k  - {M}_*  \right)  V_k^{1/2}  \right\|_2^2$ using the confidence region $\mathcal{C}^{(k)}_1$. 


When $k \in \left[ \left\lceil \sqrt{K} \right\rceil , K \right]$, we have 
\begin{align}
	\beta_{ k , \delta} & = 1 + 4 C^2 G_{k ,\delta}^2 {Hk} + G'_{k,\delta} \nonumber \\
    &\le 1 + 4 C^2 G_{ \left\lceil \sqrt{K} \right\rceil , \delta}^2 H \left\lceil \sqrt{K} \right\rceil + G'_{K,\delta} \label{eq:bound-beta-monotonicity} \\
    &= \mathcal{O} \left( H^{2} \log (d/ \delta) + m^2 \log (HK / \delta) \right) \label{eq:app-bound-delta-pp-part3}, 
\end{align}

where (\ref{eq:bound-beta-monotonicity}) is due to $G_{k ,\delta}^2 {Hk}$ decreases with $k$ for $k \in \left[ \left\lceil \sqrt{K} \right\rceil, K \right]$, and $G_{k ,\delta}' $ increases with $k$ for $k \in \left[ \left\lceil \sqrt{K} \right\rceil, K \right]$. 

Combining the above results (\ref{eq:app-bound-delta-pp-part1}), (\ref{eq:app-bound-delta-pp-part2}), (\ref{eq:app-projection-dummy-1}), (\ref{eq:app-bound-delta-pp-part3}) yields the lemma statement. 
\end{proof}


\subsection{} \label{app:bound-delta-pp-2}

\begin{lemma}
    \label{lem:bound-delta-pp-2}
	Assume event $\mathcal{E}_{K,\delta}$ holds. Under Assumption \ref{assumption:bounds}, we have 
	\begin{align*}
		\hspace{-4pt} \sum_{k=\sqrt{K}}^K \sum_{h=1}^H \left\|  V_{k} ^{-1/2}  \hspace{-0pt}  \hat{z}_{k,h} \right\|_2^2 \hspace{-2pt} \le 
		\widetilde{\mathcal{O}} \left( mH + H^2 \right), 
	\end{align*}
	where $V_{k} ^{-1/2}  := \left( V_{k}^{\dagger} \right)^{1/2}$. 
\end{lemma}

\begin{proof}
	Recall $V_k := P_k^{\mathrm{aug}} \widetilde{V}_k P_k^{\mathrm{aug}}$. 
	We have
	\begin{align} 
		 & \left\| {V}_k^{-1/2} \hat{z}_{k,h} \right\|_2^2 \nonumber \\ 
		 =& \left\| P_k^{\mathrm{aug}} \widetilde{V}_k^{-1/2}  P_k^{\mathrm{aug}}  \hat{z}_{k,h} \right\|_2^2 \nonumber \\  
		 =& \bigg\| P_k^{\mathrm{aug}} \widetilde{V}_k^{-1/2}  P_k^{\mathrm{aug}} \hat{z}_{k,h}  -  P_K^{\mathrm{aug}} \widetilde{V}_k^{-1/2}  P_k^{\mathrm{aug}} \hat{z}_{k,h}  +  P_K^{\mathrm{aug}} \widetilde{V}_k^{-1/2}  P_k^{\mathrm{aug}} \hat{z}_{k,h} \nonumber \\
		 &-  P_K^{\mathrm{aug}} \widetilde{V}_k^{-1/2}  P_K^{\mathrm{aug}} \hat{z}_{k,h}  +  P_K^{\mathrm{aug}} \widetilde{V}_k^{-1/2}  P_K^{\mathrm{aug}} \hat{z}_{k,h}  \bigg\|_2^2  \nonumber \\
		 \le& \left\| \left( P_k^{\mathrm{aug}} - P_K^{\mathrm{aug}} \right) \widetilde{V}_k^{-1/2} P_k^{\mathrm{aug}} \hat{z}_{k,h}  \right\|_2^2 + 
		 \left\| P_K^{\mathrm{aug}} \widetilde{V}_k^{-1/2}  \left( P_k^{\mathrm{aug}} - P_K^{\mathrm{aug}} \right) \hat{z}_{k,h}  \right\|_2^2 \nonumber \\
		 &+
		 \left\| P_K^{\mathrm{aug}} \widetilde{V}_k^{-1/2}  P_K^{\mathrm{aug}} \hat{z}_{k,h} \right\|_2^2, \label{eq:app-decompose-proj}
	\end{align}
	where (the sum over) the third term is bounded by Lemma \ref{lem:sum-VKk}.

	Also, under event $\mathcal{E}_{K,\delta}$, Lemma \ref{lem:projection} tells us
	\begin{align*}
		& \sum_{k= \left\lceil \sqrt{K} \right\rceil + 1 }^K \sum_{h=1}^{H-1} \left\|  P_K^{\mathrm{aug}} - P_* \right\|_2^2 
		= {\mathcal{ O }} \left( H^2 \log (K /\delta) \right), \\
		& \sum_{k= \left\lceil \sqrt{K} \right\rceil + 1 }^K \sum_{h=1}^{H-1} \left\|  P_k^{\mathrm{aug}} - P_* \right\|_2^2 = {\mathcal{ O }} \left( H^2 \log^2 (K /\delta) \right), 
	\end{align*} 
	which gives 
	\begin{align*} 
		&  \sum_{k= \left\lceil \sqrt{K} \right\rceil + 1 }^K \sum_{h=1}^{H-1} \left\|  P_K^{\mathrm{aug}} - P_k \right\|_2^2 = \widetilde{\mathcal{ O }} \left( H^2  \right). 
	\end{align*} 
	
	This means
	\begin{align*}
	    &\sum_{k= \left\lceil \sqrt{K} \right\rceil + 1 }^K \sum_{h=1}^{H-1} \left\| \left( P_k^{\mathrm{aug}} - P_K^{\mathrm{aug}} \right) \widetilde{V}_k^{-1/2} P_k^{\mathrm{aug}} \hat{z}_{k,h}  \right\|_2^2  \\
		\le& \sum_{k= \left\lceil \sqrt{K} \right\rceil + 1 }^K \sum_{h=1}^{H-1} \left\|  P_K^{\mathrm{aug}} - P_* \right\|_2^2 = {\mathcal{ O }} \left( H^2 \log (K /\delta) \right), \\
		&  \sum_{k= \left\lceil \sqrt{K} \right\rceil + 1 }^K \sum_{h=1}^{H-1} \left\| P_K^{\mathrm{aug}} \widetilde{V}_k^{-1/2}  \left( P_k^{\mathrm{aug}} - P_K^{\mathrm{aug}} \right) \hat{z}_{k,h}  \right\|_2^2 \\
		\le&  \sum_{k= \left\lceil \sqrt{K} \right\rceil + 1 }^K \sum_{h=1}^{H-1} \left\|  P_k^{\mathrm{aug}} - P_* \right\|_2^2 = {\mathcal{ O }} \left( H^2 \log^2 (K /\delta) \right), 
	\end{align*} 
	
	We can apply the above results to (\ref{eq:app-decompose-proj}), and use Lemma \ref{lem:sum-VKk} (proved after this) to get
	\begin{align*} 
		 \sum_{k= \left\lceil \sqrt{K} \right\rceil + 1 }^K \sum_{h=1}^H \left\|  V_{k} ^{-1/2}  \hspace{-0pt} \hat{z}_{k,h} \right\|_2^2 \hspace{-2pt} \le 
		\widetilde{\mathcal{O}} \left( m H + H^2 \right). 
	\end{align*} 
\end{proof}

\begin{lemma} \label{lem:sum-VKk}
	Recall 
	\begin{align*}
		\widetilde{V}_{{k}} := \Vtil{k} . 
	\end{align*}
	For $k \in [1,K]$, let $V_{K,k} := P_K^{\mathrm{aug}} \widetilde{V}_{k} P_K^{\mathrm{aug}} $. Then under Assumption \ref{assumption:bounds}, we have 
	\begin{align*}
		 \sum_{k=1}^K \sum_{h=1}^H \left\|  V_{K,k} ^{-1/2}  \hat{z}_{k,h} \right\|_2^2 \hspace{-2pt} \le 
		\mathcal{O} \left( m H \log (HK) \right), 
	\end{align*}
	where $V_{K,k} ^{-1/2}  := \left( V_{K,k}^{\dagger} \right)^{1/2}$. 
\end{lemma}

\begin{proof}
	Notationally, for any matrix $ S $, we will use $ S^{-1/2} $ to denote $ \left(S^\dagger \right)^{1/2} $ (when the matrix $S$ is not invertible. )
	Recall: $\widetilde{V}_{K} :=  I_{d+d_u} +\sum_{h=1}^{H-1} \sum_{k^\prime = 1}^{K-1} \hat{z}_{k^\prime,h} \hat{z}_{k^\prime,h} $. 
	
	For the projection matrix $P_K^{aug}$, let $L_{K}^{\mathrm{aug}}$ be the matrix of $m$ orthonormal columns such that $ P_K^{aug} = L_{K}^{\mathrm{aug}} \left( L_{K}^{\mathrm{aug}} \right)^\top  $. 
	
	Define $ D_{K,k} :=  \left( L_{K}^{\mathrm{aug}} \right)^\top  \widetilde{V}_{k} L_{K}^{\mathrm{aug}}$. 
	
	Let $\det^*$ be the pseudo-determinant operator, since $ P_K^{\mathrm{aug}} \widetilde{V}_{k} P_K^{\mathrm{aug}}  $ is a PSD matrix, ${\det}^* P_K^{\mathrm{aug}} \widetilde{V}_{k} P_K^{\mathrm{aug}} $ is the product of positive eigenvalues of $ P_K^{\mathrm{aug}} \widetilde{V}_{k} P_K^{\mathrm{aug}} $, which is the determinant of $ \left( L_{K}^{\mathrm{aug}} \right)^\top  \widetilde{V}_{k} L_{K}^{\mathrm{aug}} $.

	Thus for $1 \le k \le K$
	\begin{align}
		 &{\det}^* P_K^{\mathrm{aug}} \widetilde{V}_{k} P_K^{\mathrm{aug}} 
		= {\det} \left( \left( L_{K}^{\mathrm{aug}} \right)^\top  \widetilde{V}_{k} L_{K}^{\mathrm{aug}} \right)  = {\det} \left( \left( L_{K}^{\mathrm{aug}} \right)^\top \left( \widetilde{V}_{k-1} + \sum_{h=1}^{H-1} \hat{z}_{k-1,h} \hat{z}_{k-1,h}^\top  \right) L_{K}^{\mathrm{aug}} \right).  \label{eq:app-det-dummy1}
	\end{align}
	Define 
	\begin{align*}
		 \Omega_{K,k}  :=  D_{K,k-1}^{-1/2} \left( L_{K}^{\mathrm{aug}} \right)^{ \top}  \left(  \sum_{h=1}^{H-1} \hat{z}_{k-1,h} \hat{z}_{k-1,h}^\top  \right)  L_{K}^{\mathrm{aug}} D_{K,k-1}^{-1/2} . 
	\end{align*}
	Then (\ref{eq:app-det-dummy1}) can be written as
	\begin{align}
		&{\det}^* P_K^{\mathrm{aug}} \widetilde{V}_{k} P_K^{\mathrm{aug}} = {\det} \left[ D_{K,k-1}^{1/2}  \left( I_{m+ d'} +  \Omega_{K,k}  \right)  D_{K,k-1} ^{1/2} \right] = {\det} \left( D_{K,k-1} \right) \det  \left( I_{m+ d'} +  \Omega_{K,k}  \right) .  
		\label{eq:sum-of-Vz-internal} 
	\end{align}

	Let $\{\lambda_i\}_{i = 1,2,\cdots,m+d' }$ be the eigenvalues of 
	$  I_{m + d' } +  \Omega_{K,k} $, which are at least $1$, since $\Omega_{K,k}$ are positive semi-definite. Let $ V_{K,k} := P_K^{\mathrm{aug}} \tilde{V}_k P_K^{\mathrm{aug}} $. From this definition, we have 
	\begin{align}
		 V_{K,k} = P_K^{\mathrm{aug}} V_{K,k} P_K^{\mathrm{aug}} = L_{K}^{\mathrm{aug}} D_{K,k} \left( L_{K}^{\mathrm{aug}} \right)^\top \label{eq:relate-V-D}
	\end{align}
	For any $k, 1\le k \le K$, we have 
	\begin{align*}
		&\sum_{i=1}^{m + d' } \left( \lambda_i - 1\right) \\
		=& tr \left(   I_{m+ d'} + \Omega_{K,k} \right)  - m - d'  \\
		=&  \sum_{h=1}^{H-1} tr \hspace{-2pt} \left[ D_{K,k-1}^{-1/2} \left( L_{K}^{\mathrm{aug}} \right)^\top \hspace{-4pt}  \hat{z}_{k-1,h} \hat{z}_{k-1,h}^\top \hspace{-0pt}  L_{K}^{\mathrm{aug}} D_{K,k-1}^{-1/2} \right], 
	\end{align*}
	where the last line uses definition of $\Omega_{K,k-1}$ and linearity of trace operator. Next, by circular property of trace operator, 
	\begin{align}
		&\sum_{i=1}^{m + d' } \left( \lambda_i - 1\right) \nonumber \\
		=&  \sum_{h=1}^{H-1} tr \hspace{-2pt} \left[ L_{K}^{\mathrm{aug}}  D_{K,k-1}^{-1} \left( L_{K}^{\mathrm{aug}} \right)^\top \hspace{-4pt}  \hat{z}_{k-1,h} \hat{z}_{k-1,h}^\top \hspace{-0pt}   \right] \nonumber \\
		=&  \sum_{h=1}^{H-1} tr \hspace{-0pt} \left[  \widetilde{V}_{K,k-1}^{\dagger}    \hat{z}_{k-1,h} \hat{z}_{k-1,h}^\top \hspace{-0pt}   \right] 
		\nonumber \\
		=& \sum_{h=1}^{H-1} \left\|  V_{K,k-1}^{-1/2}  \hat{z}_{k-1,h} \right\|_2^2
		.\nonumber
	\end{align}

	Putting the above results together, we have 
	
	\begin{align*}
		& { \det} \left( I_{m + d' } + \Omega_{K,k}  \right) = \prod_{i=1}^{m + d' } \lambda_i = \prod_{i=1}^{m + d' } [ ( \lambda_i - 1) + 1 ] \\
		\ge& 1 + \sum_{i=1}^{m + d' } (\lambda_i - 1) = 1 +  \sum_{h=1}^{H-1} \left\|  V_{K, k-1}^{-1/2} \hat{z}_{k-1,h} \right\|_2^2, 
	\end{align*}
	
	which gives 
	\begin{align*}
		&1 + \sum_{h=1}^{H-1} \left\|  V_{K,k-1}^{-1/2}  \hat{z}_{k-1,h} \right\|_2^2 \\
		=& \frac{\sum_{h=1}^{H-1} \left( 1 + (H-1) \left\|  V_{K,k-1}^{-1/2}  \hat{z}_{k-1,h} \right\|_2^2 \right) }{H-1} \\
		\ge& \prod_{h=1}^{H-1} \left( 1 + (H-1) \left\|  V_{K,k-1}^{-1/2}  \hat{z}_{k-1,h} \right\|_2^2 \right)^{1/(H-1)} 
	\end{align*}

	Put the above results into (\ref{eq:sum-of-Vz-internal}) to get
	
	\begin{align}
		 &{\det}^* P_K^{\mathrm{aug}} \widetilde{V}_{K} P_K^{\mathrm{aug}} 
		= {\det} \left( L_{K}^{\mathrm{aug}} \right)^\top \widetilde{V}_{K} L_{K}^{\mathrm{aug}} \nonumber \\
		=& {\det} \left( D_{K,K-1} \right) \det  \left( I_{m + d'} +  \Omega_{K,k-1}  \right) \nonumber \\
		\ge& {\det} \left( D_{K,K-1} \right) \prod_{h=1}^{H-1} \left( 1 + \left\|  V_{K,K-1}^{-1/2} \hat{z}_{k-1,h} \right\|_2^2 \right)^{1/(H-1)} \nonumber \\
		\ge& \cdots \nonumber \\
		\ge& \prod_{k = 1}^{K-1} \prod_{h=1}^{H-1} \left( 1 + \left\|  V_{K,k-1}^{-1/2} \hat{z}_{k-1,h} \right\|_2^2 \right)^{1/(H-1)} \label{eq:app-for-log-to-det-sum-VKk} 
	\end{align} 
	
	Thus, under Assumption \ref{assumption:bounds}, we have 
	
	\begin{align}
		\sum_{k=1}^{K-1} \sum_{h=1}^{H-1} \left\| V_{K,k}^{-1/2} \hat{z}_{k,h} \right\|_2^2 \nonumber \le& C^2 \sum_{k=1}^{K-1} \sum_{h=1}^{H-1} \min \left\{ 1, \left\| V_{K,k}^{-1/2} \hat{z}_{k,h} \right\|_2^2 \right\} \nonumber \\
		\le& 2 C^2 \sum_{k=1}^{K-1} \sum_{h=1}^{H-1} \log \left( 1 + \left\| V_{K,k}^{-1/2} \hat{z}_{k,h} \right\|_2^2 \right) \nonumber \\
		=& 2 C^2  \log \prod_{k=1}^{K-1} \prod_{h=1}^{H-1} \left( 1 + \left\| V_{K,k}^{-1/2} \hat{z}_{k,h} \right\|_2^2 \right) \nonumber \\
		\le& 2C^2 (H-1) \log {\det}^* \left( P_K^{\mathrm{aug}} \widetilde{V}_{K} P_K^{\mathrm{aug}} \right) \label{eq:use-above-res-in-app-sum-Vk} \\
		\le& 2C^2 (H-1) \log \left( \frac{tr \left( P_K^{\mathrm{aug}} \widetilde{V}_{K} P_K^{\mathrm{aug}} \right) }{ m + d_u } \right)^{m + d_u } \label{eq:use-CS-in-app-sum-Vk} \\
		\le& 2C^2 H (m + d_u ) \log \left( 1 + \frac{C  HK }{ m + d_u } \right),
	\end{align} 
	where (\ref{eq:use-above-res-in-app-sum-Vk}) uses (\ref{eq:app-for-log-to-det-sum-VKk}), and (\ref{eq:use-CS-in-app-sum-Vk}) uses Cauchy-Swarchz inequality. Finally, since $m = \Theta (d_u)$, we arrive at the lemma statement. 
\end{proof}

\section{Proof of Theorem \ref{thm}}

\begin{theoremn}[\ref{thm}]
	Under Assumptions \ref{assumption:bounds}-\ref{assumption:state-eigenvalue}, for any $\delta > 0$, with probability at least $1 - (4 K + 2) \delta$, the regret for the first $K$ ($K > K_{\min }^2$) rounds satisfies 
	\begin{align}
		&Reg (K) \le \widetilde{\mathcal{O}} \left( \left( H^{5/2} + m^{3/2} H \right) \sqrt{K}  \log \left( K / \delta \right) \right), 
	\end{align}
	where $\mathcal{O}$ omits (poly)-logarithmic terms. 
\end{theoremn}

\begin{proof}
	Recall by Proposition \ref{prop:decomp-in-app} we have 
	\begin{align*}
	    \mathrm{Reg}(K) \le \mathcal{O} \left( H \sqrt{K} \right) + \sum_{k=\left\lceil \sqrt{K} \right\rceil+ 1}^K \sum_{h=1}^{H-1}  \left( \Delta_{k,h} + \Delta_{k,h}^\prime + \Delta_{k,h}^{\prime \prime } \right)
	\end{align*}
	By Lemma \ref{lem:regret-simple-terms}, we know, condition on $\mathcal{E}_{K,\delta}$ ($K > K_{\min}^2$) being true, with probability at least $1- 2\delta$, 
	\begin{align*} 
			\sum_{k= \left\lceil \sqrt{K} \right\rceil +1 }^{K} \sum_{h=1}^{H-1} \left( \Delta_{k,h} + \Delta_{k,h}^\prime  \right) \le& \mathcal{O} \left( \sqrt{ K H^3 \log \frac{2}{\delta} } + \sqrt{ 2HK  \log \frac{2}{\delta} } \right). 
	\end{align*}
	Next, from Lemma \ref{lem:bound-delta-pp} we know, under event $\mathcal{E}_{K,\delta}$, 
	\begin{align*}
			& \sum_{k=\left\lceil \sqrt{K} \right\rceil + 1 }^K \sum_{h=1}^{H-1} \Delta_{k,h}'' \le \widetilde{\mathcal{O}} \left( \left( H^{5/2} + m^{3/2} H \right)  \sqrt{K}  \right).  
	\end{align*} 
	The above facts together prove the theorem statement. 
\end{proof}

\section{Experiment Details} \label{app:exp}




The state features fed to the systems are (down-sampled) image representations of the state, together with cart velocity and pole tip velocity. The pole tip velocity is clipped to within a certain range. The $Q$ matrix penalizes large velocities, bad cart position and bad pole positions. Since the states contain image information, the state space is high dimensional, while the image clearly has a low-rank representation. A random search is used when finding the optimistic estimate $\widetilde{M}_k$. 

All experiments are conducted on a \texttt{Dell Precision T3620 Mini Tower} machine with the following configuration. 
\begin{itemize}
    \item Processor: 6th Gen Intel Core i7-6700 (Quad Core 3.40GHz, 4.0Ghz Turbo, 8MB).
    \item Memory: 32GB (4x8GB) 2133MHz DDR4 Non-ECC.
    \item Video Card: NVIDIA Quadro K620, 2GB.
    \item HDD: 256GB SATA Class 20 Solid State Drive.
\end{itemize}

For software environment, the experiments are conducted using Python 3.6.3 and \texttt{OpenAI Gym 0.15.4} \cite{1606.01540}.

\end{document}